\documentclass[lettersize,journal]{IEEEtran}
\usepackage{amsmath,amsfonts}
\usepackage{algorithmic}
\usepackage{array}
\usepackage{textcomp}
\usepackage{stfloats}
\usepackage{url}
\usepackage{verbatim}
\usepackage{graphicx}

\usepackage[ruled, vlined, linesnumbered]{algorithm2e}
\usepackage{cite}
\usepackage{amssymb,amsfonts,bm}
\usepackage{amsmath,tikz}
\usepackage{color}
\usepackage{mathtools}
\usepackage[hidelinks]{hyperref} 
\usepackage{rotating}
\usepackage{blkarray}
\usepackage{physics}
\usetikzlibrary{arrows}
\let\proof\relax
\let\endproof\relax
\usepackage{amsthm}
\usepackage{comment}

\newtheorem{thm}{Theorem}[section]

\newtheorem{prop}[thm]{Proposition}
\newtheorem{rem}[thm]{Remark}

\newtheorem{ass}[thm]{Assumption}

\usepackage{caption}
\usepackage{subcaption}
\usepackage{etoolbox}
\let\classAND\AND
\let\AND\relax
\let\AND\classAND
\AtBeginEnvironment{algorithmic}{\let\AND\algoAND}

\graphicspath{{eps/}{png/}}
\DeclareSymbolFont{symbolsC}{U}{pxsyc}{m}{n}
\DeclareMathSymbol{\coloneqq}{\mathrel}{symbolsC}{"42}

\newcommand{\vertiii}[1]{{\left\vert\kern-0.25ex\left\vert\kern-0.25ex\left\vert
		#1 
		\right\vert\kern-0.25ex\right\vert\kern-0.25ex\right\vert}}
\usepackage{multirow}
\usepackage{makecell}
\usepackage{adjustbox}
\usepackage{float}

\newcommand*{\vecbf}[1]{\mathbf{#1}} % vector boldface
\newcommand*{\gbf}[1]{\bm{#1}} % greek boldface
\newcommand*{\myset}[1]{\mathcal{#1}} % greek boldface
\newcommand*{\lrangle}[1]{\langle #1 \rangle} 
\newcommand*{\aug}[1]{\hat{#1}} % vector boldface
\newcommand*{\maxop}[1]{\bar{#1}} % vector boldface
\newcommand*{\minop}[1]{\underline{#1}} % vector boldface
\newcommand{\init}[1]{{#1}^{\text{in}}} 
\newcommand{\en}[1]{{#1}^{\text{end}}} 
\newcommand{\back}[1]{{#1}_{\text{bac}}} 
\newcommand{\front}[1]{{#1}_{\text{fro}}} 
\newcommand{\overbar}[1]{\mkern 1.5mu\overline{\mkern-1.5mu#1\mkern-1.5mu}\mkern 1.5mu}
\newcommand{\trajorder}{q} 
\newcommand{\numcon}{\gamma} 
\newcommand{\textornot}[1]{{#1}} 
\newcommand{\etal}{\textit{et al.}}
\newcommand{\Obset}{\mathbf{O}}
\newcommand{\bigO}{\mathcal{O}}

\newcolumntype{M}[1]{>{\centering\arraybackslash}m{#1}}

\newcommand{\cb}{\color{black}}
\newcommand{\posctrlpt}{\gbf{\beta}}
\newcommand{\ctrlptbspline}{\beta}

\usepackage{pifont}% http://ctan.org/pkg/pifont
\usepackage{soul}
\SetKwFor{While}{while}{}{end while}%
\DeclareSymbolFont{textsymbols}{TS1}{\familydefault}{m}{n}
\SetSymbolFont{textsymbols}{bold}{TS1}{\familydefault}{m}{n}
\DeclareMathSymbol{\ulq}{\mathopen}{textsymbols}{39}

\def\BibTeX{{\rm B\kern-.05em{\sc i\kern-.025em b}\kern-.08em
    T\kern-.1667em\lower.7ex\hbox{E}\kern-.125emX}}
\usepackage{balance}
\begin{document}
\title{NEPTUNE: Nonentangling {\cb Trajectory} Planning for Multiple Tethered Unmanned Vehicles}
\author{Muqing~Cao, Kun~Cao, Shenghai~Yuan, Thien-Minh~Nguyen, %~\IEEEmembership{Member,~IEEE},
and~Lihua~Xie,~\IEEEmembership{Fellow,~IEEE}
\thanks{The authors are with School of Electrical and Electronic Engineering, Nanyang Technological University, 50 Nanyang Avenue, Singapore 639798
(e-mail: caom0006@e.ntu.edu.sg, kun001@e.ntu.edu.sg, shyuan@ntu.edu.sg, thienminh.nguyen@ntu.edu.sg, elhxie@ntu.edu.sg).
}
\thanks{Corresponding author: Lihua Xie.}
%\thanks{M. Cao, K. Cao, S. Yuan T.M. Nguyen and L. Xie (corresponding author, elhxie@ntu.edu.sg) are with School of Electrical and Electronic Engineering, Nanyang Technological University, 50 Nanyang Avenue, Singapore 639798. }
\thanks{This paper has supplementary downloadable multimedia material.}
}

\markboth{}%
{How to Use the IEEEtran \LaTeX \ Templates}

\maketitle

\begin{abstract}
\textcolor{black}{Despite recent progress in trajectory planning for multiple robots and a single tethered robot, trajectory planning for multiple tethered robots to reach their individual targets without entanglements remains a challenging problem.}
In this paper, a complete approach is presented to address this problem.
{\color{black}First, a multi-robot tether-aware representation of homotopy is proposed to efficiently evaluate the feasibility and safety of a potential path in terms of}
(1) the cable length required to reach a target following the path, and (2) the risk of entanglements with the cables of other robots.
Then the proposed representation is applied in a decentralized and online planning framework, which includes a graph-based kinodynamic trajectory finder and an optimization-based trajectory refinement, to generate entanglement-free, collision-free, and dynamically feasible trajectories.
{\color{black}The efficiency of the proposed homotopy representation is compared against the existing single and multiple tethered robot planning approaches.}
{\color{black}Simulations with up to 8 UAVs show the
effectiveness of the approach in entanglement prevention and its
real-time capabilities.}
Flight experiments using 3 tethered UAVs verify the practicality of the presented approach.
{\color{black}The software implementation is publicly available online\footnote{https://github.com/caomuqing/neptune}.}
\end{abstract}

\begin{IEEEkeywords}
Tethered Robot Planning, Multi-Robot, Trajectory Planning.
\end{IEEEkeywords}

\section{Introduction}\label{sec:intro}
\IEEEPARstart{U}{nmanned} vehicles such as unmanned aerial vehicles (UAVs), unmanned ground vehicles (UGVs) and unmanned surface vehicles (USVs) have been widely adopted in industrial applications due to reduced safety hazards for humans and potential cost saving \cite{doi:10.1142/S2301385020500089, GALCERAN20131258}. Tethered systems are commonly employed to extend the working duration, enhance the communication quality and prevent the loss of unmanned vehicles. For autonomous tethered robots, {\color{black}it is important to consider the risk of the tether being entangled with the surroundings,} which will limit the reachable space of the robots and even cause damage. 

In this work, we consider the trajectory planning problem for multiple tethered robots in a known workspace with static obstacles. 
{\color{black}Each robot is attached to one end of a slack and flexible cable that is allowed to lie on the ground.} 
The other end of the cable is attached to a fixed base station. 
{\color{black}The cable has a low-friction surface so that it can slide over the surface of static obstacles or other cables.
Entanglement occurs when the movement of at least one of the robots is restricted due to the physical interactions among the cables.
In the scenario shown in Figure \ref{fig: entangle1g}, two ground robots' cables cross each other.}
If the robots continue to move in the directions indicated by the arrows, {\color{black}the cables will be stressed, thus affecting the movement of at least one of the robots.
Such a situation is more likely to occur} when more robots operate in the same workspace.
\begin{figure}[!t]
\centering
\includegraphics[width=0.8\linewidth]{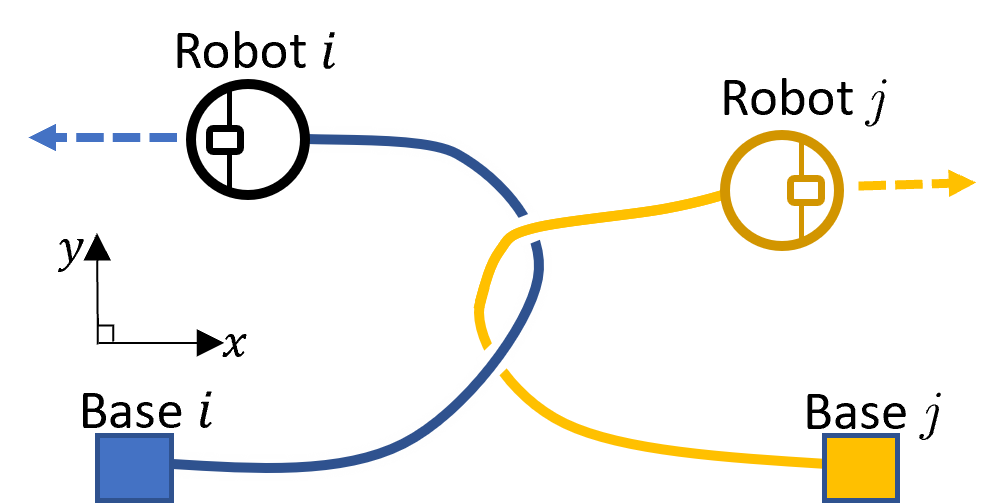}
\caption{\footnotesize Top-down view of a workspace to illustrate an entanglement situation. Note the Z-order of the cables (shown as blue and yellow curves) at their two intersections.}
\label{fig: entangle1g}
\end{figure}

{\color{black}While there exists abundant literature on multi-robot path and trajectory planning, 
they are not applicable to tethered multi-robot scenarios \cite{liu2018towards,zhou2021ego,zhou2021decentralized,wang2021geometrically,tordesillas2021mader}. 
Most of the studies on the tethered robot planning problem focus on the single robot case and use a representation of homotopy to identify the path or cable configuration \cite{kim2014path,kim2015path,salzman2015optimal,Aine2016}.}
Feasible paths are found by searching in a graph augmented with the homotopy classes of the paths.
However, the existing representation of homotopy lacks the capability of representing the interaction of multiple mobile robots efficiently. {\color{black}Moreover, slow graph expansion requires offline construction of the graph prior to online planning.
Existing studies on path planning for multiple tethered robots present centralized and offline approaches without taking static obstacles into account \cite{sinden1990tethered,hert1996ties,hert1999motion,zhang2019planning}.}

In this work, we present NEPTUNE,
a decentralized and online trajectory generation framework for 
\underline{n}on-\underline{e}ntangling trajectory \underline{p}lanning for multiple \underline{t}ethered \underline{un}manned v\underline{e}hicles.
First, we propose a novel multi-robot tether-aware representation of homotopy that 
encodes the interaction among the cables of the planning robot and the collaborating robots, and the static obstacles.
{\color{black} By this representation, the risk of entanglement with other robots and static obstacles is efficiently evaluated.
Furthermore, the proposed representation enables efficient determination of the reachability of a destination under the given tether length constraints.
}
{\color{black}The trajectory planning consists of a front-end trajectory search module as well as a back-end optimization module.} The front end searches for a feasible, collision-free, non-entangling, and goal-reaching polynomial trajectory, using kinodynamic A* in a graph augmented with the introduced multi-robot tether-aware representation.
The back-end trajectory optimization refines the first few segments of the feasible trajectory to generate a trajectory with lower control effort while still satisfying the non-collision and non-entangling requirements.
Each robot generates its own trajectory in a decentralized and asynchronous manner and broadcasts the future trajectory through a local network for others to access.

{\color{black}To the best of our knowledge, NEPTUNE is the first online and decentralized trajectory planner for multiple tethered robots in an obstacle-ridden environment.}
The main contributions of this paper are summarized as follows:
\begin{itemize}
	\item {\color{black}A detailed procedure to obtain a multi-robot tether-aware representation of homotopy is presented}, which enables efficient checks on the risk of entanglement, as well as efficient computation of the required cable length to reach a target.
	\item A complete tether-aware planning framework is presented consisting of a kinodynamic trajectory finder and a trajectory optimizer. 
	\item {\color{black}Comparisons with the existing approaches for tethered robot path planning in single-robot obstacle-rich and multi-robot obstacle-free environments demonstrate significant improvements in computation time.}
	\item {\color{black}Simulations using $2$ to $8$ robots in an obstacle-ridden environment reveal an average computation time of less than $70$ ms and a high mission success rate.}
	\item Flight experiments using three UAVs verify the practicality of the approach. {\color{black}We make the software implementation publicly available for the benefit of the community}.
\end{itemize}

{\color{black}The type of mobile robot considered in this paper is UAV. However,} the presented approach is also applicable to other types of vehicles such as UGV and USV.

\section{Related Works}\label{sec:related}
\subsection{Tethered Robot Path Planning}
Interestingly, {\color{black}most of the early works on the tethered robot planning problem focus on multiple robots rather than a single robot.} Sinden \cite{sinden1990tethered} investigated the scheduling of tethered robots to visit a set of pre-defined locations in turns so that none of the cables crosses each other during the motion of the robots. A bipartite graph is constructed, with colored edges representing the ordered cable configurations.
In \cite{hert1996ties,hert1999motion}, the authors addressed the path planning for tethered robots with specified target cable configurations. A directed graph is used to represent the motion constraints and the ordering of movements. 
The output is a piece-wise linear path for each robot and waiting time at some specified location.
Zhang \etal \cite{zhang2019planning} extended the results of \cite{hert1996ties} by providing an analysis of a more efficient motion profile where all robots move straight and concurrently.
These works are very different from our work in terms of problem formulation and approach: (1) {\cb they consider taut cables forming straight lines between robots and bases, whereas we consider slack cables that are allowed to slide over one another;} (2) static obstacles are not taken into account in these works; (3) their approaches are offline and centralized while our work presents a decentralized online approach; (4) the outputs of these algorithms are piece-wise linear paths whereas our approach provides dynamically feasible trajectories.

The development of planning algorithms for a single tethered robot typically focuses on navigating the robot around obstacles to reach a goal while satisfying the cable length constraint.
Early work \cite{xavier1999shortest} and its recent derivative \cite{brass2015shortest} find the shortest paths in a known polygonal environment by tracing back along the previous path to look for turning points in a visibility graph-like construction.
Recent developments of homotopic path planning using graph-search techniques \cite{igarashi2010homotopic,bhattacharya2012topological,hernandez2015comparison,bhattacharya2018path} provide the foundation for a series of new works on tethered robot path planning.
Particularly, Kim \etal \cite{kim2014path} use a homotopy invariant (h-signature) to determine the homotopy classes of paths by constructing a word for each path (Section \ref{sec: prelim}).
{\color{black}A homotopy augmented graph is built with the graph nodes carrying both a geometric location and the homotopy class of a path leading to the location.} Then graph search techniques can be applied to find the optimal path subject to grid resolution.
Kim \etal \cite{kim2015path} and Salzman \etal \cite{salzman2015optimal} improve the graph search and graph building processes of \cite{kim2014path} respectively through applying a multi-heuristic A* \cite{Aine2016} search algorithm and replacing the grid-based graph with a visibility graph. McCammon \etal \cite{mccammon2017planning} extend the results to a multi-point routing problem.
The homotopy invariant in these works, which is for a 2-D static environment, is insufficient to represent the complex interactions when multiple tethered robots are involved (a justification will be provided in Section \ref{subsec: homotopy}).
Furthermore, these works use a curve shortening technique 
to determine whether an expanded node satisfies the cable length constraint, 
which is computationally expensive and leads to slow graph expansion.
It is thus a common practice to construct the augmented graph in advance.
Bhattacharya \etal \cite{bhattacharya2018path} proposed a homotopy invariant for multi-robot coordination, which can be potentially applied to the centralized planning for tethered multi-robot tasks.
However, {\color{black}considering the high dimensions of the graph and the complexity related to identifying homotopy equivalent classes (the word problem), it is time-consuming to build a graph even for a simple case} (more than 30s for $3$ robots in a $7\times7$ grid).
{\color{black}Compared with these works, we develop an efficient representation of homotopy for multiple tethered robots, thus facilitating the computation of the required cable length. Therefore, the graph expansion and search can be executed online for on-demand targets.}
 Teshnizi \etal \cite{Teshnizi2014computing} proposed an online decomposition of workspace into a graph of cells for single-robot path searching. 
A new cell is created when an event of cable-cable crossing or cable-obstacle contact is detected.
However, infinite friction between cable surfaces is assumed, which deviates from the practical cable model.
Furthermore, a large number of cells can be expected during the graph search in an obstacle-rich environment. {\color{black}The reason is that a new cell is created for each visible vertex while no heuristics are provided for choosing the candidate cells.}

{\cb Several studies have been recently performed which }leverage Braid groups to characterize the topological patterns of robots' trajectories and facilitate the trajectory planning for multiple robots \cite{Diaz2017multirobot, Mavrogiannis2020, mavrogiannis2022analyzing}.
Despite promising results, 
these works have not been extended to the trajectory planning for tethered robots.

\subsection{Decentralized Multi-robot trajectory planning}
{\cb Most of the existing works on the tethered robot planning problem generate piece-wise linear paths.} In this section, we review some smooth trajectory generation techniques for decentralized multi-robot planning, {\cb while focusing on methods applied to UAVs.}
In general, decentralized multi-robot trajectory generation includes synchronous and asynchronous methods.
Synchronous methods such as \cite{chen2015decoupled,luis2019trajectory} require trajectories to be generated at the same planning horizon for all robots, whereas asynchronous methods do not have such a restriction and hence are more suitable for online application.
In \cite{liu2018towards}, the authors presented a search-based multi-UAV trajectory planning method, where candidate polynomial trajectories are generated by applying discretized control inputs. {\cb Trajectories that violate the collision constraint (resulting in a non-empty intersection between robots' polygonal shapes) are discarded.}
In \cite{zhou2021ego}, the collision-free requirement is enforced as a penalty function in the overall objective function of the nonlinear optimization problem.
In \cite{zhou2021decentralized}, a similar approach is taken, but a newly developed trajectory representation \cite{wang2021geometrically} is employed (as compared to B-spline in \cite{zhou2021ego}) {\cb which enables concurrent optimization of both the spatial and the temporal parameters.}
In \cite{tordesillas2021mader}, robots' trajectories are converted into convex hulls. The collision-free constraint is guaranteed by optimizing a set of planes separating the convex hulls of the collaborating robots and those of the planning robot.
To save computational resources and ensure short-term safety, 
an intermediate goal is chosen, which is the closest point to the goal within a planning radius.

{\color{black}As revealed by the comparisons in \cite{tordesillas2021mader,luis2019trajectory}}, centralized and offline trajectory generation approaches typically require much longer computation time to generate a feasible solution than the online decentralized approaches, and the difference grows with the number of robots involved.
Furthermore, the absence of online replanning means that the robots cannot handle in-flight events, such as the addition of a new robot into the fleet or the appearance of non-cooperative agents.
{Hence, \cb it is believed that a decentralized planner with online replanning is more suitable for practical applications. 
Furthermore,} the framework presented in this paper has the flexibility to integrate new features such as the prediction and avoidance of dynamic obstacles and non-cooperative agents.
Similar to \cite{tordesillas2021mader}, our approach uses an asynchronous planning strategy with a convex-hull representation of trajectories. 
{\color{black}However, the greedy strategy used by \cite{tordesillas2021mader} results in inefficient trajectories. The reason is that an intermediate goal may be selected near a large non-traversable region, where the robot has to make abrupt maneuvers to avoid the obstacles.
In contrast, our approach selects the intermediate goals that are derived from a feasible, goal-reaching, and efficient (based on some evaluation criteria such as length) trajectory generated by our front-end kinodynamic trajectory finder.}

\section{Preliminaries}\label{sec: prelim}
\subsection{Notation}
In this paper, $\|\vecbf{x}\|$ denotes the $2$-norm of vector $\vb{x} \in \mathbb{R}^{n}$, $\vb{x}^{(i)}$ denotes the $i$-th order derivative of vector $\vb{x}$ and $\myset{I}_{n}$ denotes the set consisting of integers $1$ to $n$, i.e., $\myset{I}_{n} = \{1,\dots,n\}$.
The notation frequently used in this paper is shown in Table \ref{tab: notation}.
More symbols will be introduced in the paper.
\begin{table}[htbp]

\def\arraystretch{1.1}
\caption{Notation}
\begin{tabular}{| M{0.19\linewidth}| m{0.702\linewidth} |}
\hline
Symbol                                                                           & Meaning                                                                                                                                                                                                                                                                          \\ \hline
n                                                                                & Number of robots.                                                                                                                                                                                                                                                                \\ \hline
m                                                                                & Number of static obstacles.                                                                                                                                                                                                                                                      \\ \hline

$\aug{\mathcal{W}}$, $\mathcal{W}$                                               & 3-dimensional workspace and its 2-D projection, i.e., $\aug{\mathcal{W}}=\{(x,y,z)|(x,y)\in{\mathcal{W}}\}$.                                                                                                                                                                     \\ \hline
{$\aug{\vb{p}}_i$, $\vb{p}_i$}                                                   & Position of robot $i$ and its 2-D projection, i.e., $\aug{\vb{p}}_i=[\vb{p}_i, p^{z}_{i}]^\top\in\mathbb{R}^3$,  $\vb{p}_i=[p_{i}^{x}, p_{i}^{y}]^\top\in\mathbb{R}^2$.                                                                                                          \\ \hline
{$\aug{\vb{p}}^{\text{term}}_i$}, $\aug{\vb{p}}^{\text{inter}}$                  & Terminal goal and intermediate goal position, $\aug{\vb{p}}^{\text{term}}_i=[p^{\text{term},\text{x}}_i,p^{\text{term},\text{y}}_i,p^{\text{term},\text{z}}_i]^\top\in\mathbb{R}^3$.                                                                                             \\ \hline
$\init{\aug{\vb{p}}}$, $\init{\aug{\vb{v}}}$, $\init{\aug{\vb{a}}}$              & The position, velocity and acceleration of the robot at time $\init{t}$, $\in\mathbb{R}^3$.                                                                                                                                                                                       \\ \hline
$\mathcal{P}{(}k{)}$                                                             & Set of robots' 2-D positions at time $k$, i.e. $\mathcal{P}(k)\coloneqq\{\vb{p}_j(k)|j\in \myset{I}_n\}$.                                                                                                                                                                        \\ \hline
$\mathcal{P}_{j,l}$                                                              & Set consisting of the positions of robot $j$ at discretized times, $\mathcal{P}_{j,l}\coloneqq\{\vb{p}_j(k)|t\in[\init{t}+lT, \init{t}+(l+\frac{1}{\sigma})T,\dots, \init{t}+(l+1)T]\}$                                                                                          \\ \hline
$\init{t}$                                                                       & The start time of the planned trajectory, referenced to a common clock.                                                                                                                                                                                                          \\ \hline
{$\maxop{\vb{v}}_i$, $\minop{\vb{v}}_i$, $\maxop{\vb{a}}_i$, $\minop{\vb{a}}_i$} & Upper and lower bounds of velocity and acceleration.                                                                                                                                                                                                                             \\ \hline
{$O_j$}                                                                          & 2-D projection of the obstacle $j$.                                                                                                                                                                                                                                              \\ \hline
{$\Obset$}                                                                       & Set of obstacles, i.e., $\Obset\coloneqq\{O_j|j\in \myset{I}_m\}$.                                                                                                                                                                                                               \\ \hline
{$\vb{b}_i$}                                                                     & {\cb 2-D base position of robot $i$}, $\vb{b}_i=[b_i^{x},b_i^{y}]^\top$.                                                                                                                                                                                                                            \\ \hline
{$\phi_i$}                                                                       & Cable length of robot $i$.                                                                                                                                                                                                                                                       \\ \hline
{$\vb{c}_{i,l}(t)$}                                                              & $l$-th contact point of robot $i$ at time $t$, $\in\mathbb{R}^2$.                                                                                                                                                                                                                 \\ \hline
{$\numcon_i(t)$}                                                                 & Number of contact points of robot $i$ at time $t$, $\in\mathbb{Z}$.                                                                                                                                                                                                              \\ \hline
$\mathcal{C}_i(k)$                                                               & List of contact points of robot $i$ at time $k$, i.e. $\mathcal{C}_i(k)=\{\vb{c}_{i,f}(k)|f\in\myset{I}_{\numcon_i(k)}\}$.                                                                                                                                                       \\ \hline
{\cb $\trajorder$}                                                                          & Order of polynomial trajectory.                                                                                                                                                                                                                                                  \\ \hline
$\aug{\mathbf{E}}_l$, $\mathbf{E}_l$                                             & $\aug{\mathbf{E}}_l\in\mathbb{R}^{(\trajorder+1)\times3}$ consists of the coefficients of a 3-dimensional $\trajorder$-th order polynomial. ${\mathbf{E}}_l\in\mathbb{R}^{(\trajorder+1)\times2}$ consists of the coefficients for the X and Y dimensions.                       \\ \hline
$\vb{g}(t)$                                                                      & The monomial basis, $\vb{g}(t)=[1, t, \dots,t^\trajorder]^\top$.                                                                                                                                                                                                                 \\ \hline
$o_{j,0}$, $o_{j,1}$                                                             & Virtual segments of obstacle $j$. {\cb The second number in the subscript indicates the index of the segment.}                                                                                                                                                                                                                                           \\ \hline
$r_{j,0}$, $r_{j,1}$                                                             & Virtual segments of robot $j$. {\cb $r_{j,1}$ represents the simplified cable configuration and $r_{j,0}$ is the extended segment from robot $j$.}                                                                                                                                                                                                                                              \\ \hline
$\gbf{\zeta}_{j,0}$, $\gbf{\zeta}_{j,1}$                                         & The two points on the surface of $O_j$ that define $o_{j,0}$ and $o_{j,1}$, respectively.                                                                                                                                                                                                          \\ \hline
$h_i(k)$                                                                           & {\cb The homotopy representation of robot $i$, expressed as a word. The subscript is omitted when no ambiguity arises.}                                                                                                                                                                                                                  \\ \hline
size$(h(k))$                                                                     & The total number of entries (letters) in the word $h(k)$.                                                                                                                                                                                                                        \\ \hline
{\cb $\text{entry}(h(k),l)$}                                                                 & The $l$-th entry in $h(k)$.                                                                                                                                                                                                                                                      \\ \hline
{\cb index$(\vb{c})$}                                                                 & {\cb The position of the obstacle entry in $h(k)$ whose surface point is $\vb{c}$, i.e., if $\vb{c}=\gbf{\zeta}_{j,f}$ for a particular $j\in\myset{I}_m$ and $f\in\{0,1\}$, then $\text{entry}(h(k),\text{index}(\vb{c}))= \textornot{o}_{j,f}$.}                                                          \\ \hline \rule{0pt}{0.3cm}
$\overbar{\vb{a}\vb{b}}$                                                         & Line segment bounded by points $\vb{a}$ and $\vb{b}$.                                                                                                                                                                                                                            \\ \hline \rule{0pt}{0.4cm}
$\overleftrightarrow{\vb{a} \vb{b}}$                                             & Line passing through $\vb{a}$ and $\vb{b}$.                                                                                                                                                                                                                                      \\ \hline
$\diamond$                                                                       & Concatenation operation.                                                                                                                                                                                                                                                         \\ \hline
$\maxop{u}$                                                                      & The maximum magnitude of control input, $>0$.                                                                                                                                                                                                                                    \\ \hline
$\sigma_\text{u}$                                                                & A user-chosen parameter such that $2\sigma_\text{u}+1$ is the number of sampled control inputs for each axis in the kinodynamic search.                                                                                                                                                                            \\ \hline
$\sigma$                                                                         & Number of discretized points for evaluating the trajectory of each planning interval.                                                                                                                                                                                            \\ \hline
$T$                                                                              & Duration of each piece of trajectory.                                                                                                                                                                                                                                            \\ \hline
$\front{\eta}$, $\back{\eta}$, $\maxop{\eta}$                                    & $\front{\eta}$ is the number of polynomial curves in the front-end output trajectory, $\back{\eta}$ is the number of polynomial curves to be optimized in the back end. $\maxop{\eta}$ is a user-chosen parameter such that $\back{\eta}=\text{min}(\front{\eta},\maxop{\eta})$. \\ \hline
{\cb $\mathcal{Q}_{l}$, $\mathcal{V}_{l}$,  $\mathcal{A}_{l}$   }                                                            & {\cb Sets of 2-D position, velocity, and acceleration control points for a trajectory with index $l$.        }                                                                                                                                                                                                          \\ \hline
{\cb $\aug{\mathcal{Q}}_{l}$, $\aug{\mathcal{V}}_{l}$, $\aug{\mathcal{A}}_{l}$} & {\cb Sets of 3-D position, velocity, and acceleration control points for a trajectory with index $l$.} \\ \hline
{\cb $\posctrlpt$, $\vb{v}$, $\vb{a}$} & {\cb The position, velocity, and acceleration control points, respectively. }\\ \hline
\end{tabular}
\label{tab: notation}
\end{table}

\subsection{Homotopy and Shortest Homotopic Path}\label{subsec: homotopy}
\begin{figure}[!t]
\centering
\includegraphics[scale=0.5]{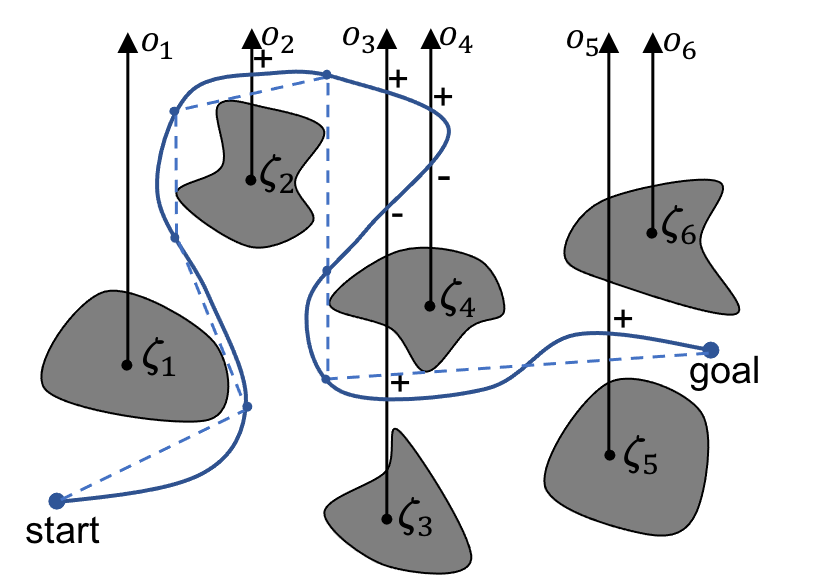}
\caption{\footnotesize 
Example of generating a homotopy invariant (h-signature) of a curve:
the solid blue path has an initial word of $ o_2o_3o_4o_4^-o_3^-o_3o_5$ which can be reduced to $ o_2o_3o_5$.
The dashed blue line is the shortened homotopy-equivalent path to the original path using the curve shortening technique in \cite{kim2014path}. {\cb It also represents the shortened cable configuration of a robot following the solid blue path, with the base coinciding with the start point.}}
\label{fig:hsig}
\end{figure}
We briefly review the concepts related to homotopy. Consider a workspace of arbitrary dimension consisting of obstacles. {\cb A path or a cable in the workspace can be represented as a curve.} Two curves in this workspace, sharing the same start and end points, are \emph{homotopic} (or belong to the same \emph{homotopy class}) if and only if one can be continuously deformed into another without traversing any obstacles.

A \emph{homotopy invariant} is a representation of homotopy that uniquely identifies the homotopy class of a curve.
In a 2-dimensional workspace consisting of $m$ obstacles, there exists a standard procedure to compute a homotopy invariant \cite{Allen2002Algebraic} (Figure \ref{fig:hsig}).
{\cb First, a set of non-intersecting rays $o_1,o_2,\dots,o_m$ are constructed, each emitting from a reference point $\gbf{\zeta}_i$ in each obstacle.
Then, a representation (word) is obtained by tracing the curve from the start to the goal and adding the corresponding letters of the crossed rays. 
Right-to-left crossings are distinguished by a superscript `-1'.}
Then, the word is reduced by canceling consecutive crossings of the same ray with opposite crossing directions.
This reduced word is called the h-signature and is a homotopy invariant. {\cb Indeed,} two curves with the same start and goal belong to the same homotopy class if and only if they have the same h-signature.

The h-signature introduced above records ray-crossing events that indicate robot-obstacle interactions. However, it is not capable of recording key events of cable-cable interaction for identifying entanglements.
{\cb To illustrate this, an h-signature is obtained for the path taken by each robot in Figure \ref{fig: entangle1g}. 
Each robot has taken a path that crosses the ray emitted from the other robot exactly once.}
{\cb Hence, the resulting h-signature records only one ray-crossing event revealing no potential entanglements.}

A tethered robot following a path from its base to a goal should have its cable configuration homotopic to the path.
Hence, a shorter homotopy-equivalent curve can be computed to approximate the cable configuration and calculate the required cable length to reach the goal. 
In \cite{kim2014path, kim2015path, salzman2015optimal}, a common curve shortening technique is used, 
which requires tracing back the original path to obtain a list of turning points,  i.e., points on the shortened path where the path changes direction. 
{\cb Each of these points' visual line of sight to its previous point does not intersect any obstacles.}
An example of this shortened curve is shown in Figure \ref{fig:hsig}.

\section{Formulation and Overview}\label{sec: formulation}
In this work, we consider a 3-dimensional, simply connected, and bounded workspace with constant vertical limits, 
$\aug{\mathcal{W}}=\{(x,y,z)|(x,y)\in{\mathcal{W}},z\in[\minop{z},\maxop{z}]\}$,
where 
$\maxop{z}$ and $\minop{z}$ are the vertical bounds.
The workspace contains $m$ obstacles, whose 2-D projections are denoted as $O_1, \dots,O_m$.
Consider a team of $n$ robots. 
Each robot $i$ is connected to a cable of length $\phi_i$,
and one end of the cable is attached to a base fixed on the ground at $\vb{b}_i$.
{\cb We reasonably assume that the bases are placed at the boundary of the workspace, not affecting the movements of the robots or the cables.}
As we mentioned in Section \ref{sec:intro}, we consider flexible and slack cables pulled toward the ground by gravity. 
In the case of a UGV, {\cb the cable lies entirely on the ground}, and in the case of a tethered UAV, the part of the cable near the UAV is lifted in the air.
{\cb A robot is allowed to cross (move over) the cables of other robots, resulting in its cable sliding over the cables of others.
The robots are not permitted to move directly on top of an obstacle, 
as this creates ambiguous cable configurations, 
where we cannot determine if the cable of a robot stays on an obstacle or falls on the ground (even if we know that the cable falls, we cannot determine which side of the obstacle it falls to).}
Such a restriction allows us to express most of the interaction of cables in the 2-D plane, regardless of the type of robots involved.

Our task is to compute trajectories for each robot in the team to reach its target $\aug{\vb{p}}^{\text{term}}_i$.
{\color{black}The trajectories should be continuous up to acceleration to prevent aggressive attitude changes when controlling the UAVs}, satisfying the velocity and acceleration constraints $\maxop{\vb{v}}_i$, $\minop{\vb{v}}_i$, $\maxop{\vb{a}}_i$, $\minop{\vb{a}}_i\in\mathbb{R}^3$, and free of collision and entanglements.

\begin{figure}[!t]
\centering
\includegraphics[width=\linewidth]{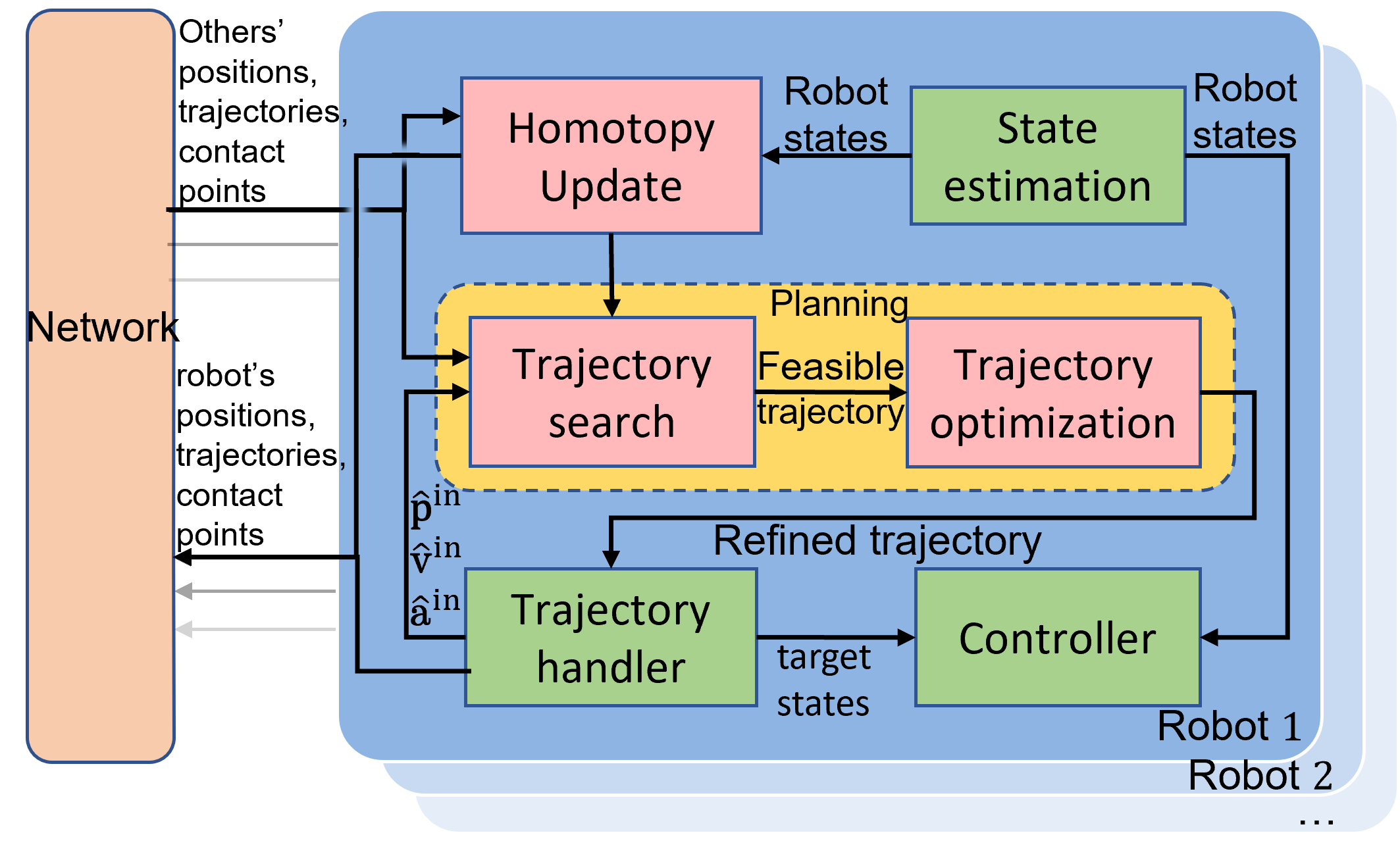}
\caption{\footnotesize 
Overview of the system. The blocks in red are the core modules, for which detailed explanations will be provided.
{\cb The state estimation module takes inputs from sensors such as IMUs, cameras, and Lidars}, which are not shown to save space. The controller module generates commands for the actuators of the robot that are also omitted in this diagram.}
\label{fig:flowchart}
\end{figure}
Figure \ref{fig:flowchart} illustrates the core components of our approach.
The robots share a communication network through which they can exchange information.
For each robot $i$, several modules are run concurrently. {\cb The Homotopy Update module maintains an updated topological status of the robot based on its current position and other robots' latest information.}
The output of this module is a representation of homotopy in the form of a word, and a list of contact points
$\vb{c}_{i,1}(t),\vb{c}_{i,2}(t),\dots,\vb{c}_{i,\numcon_i(t)}(t)\in\mathbb{R}^2$, where $\numcon_i(t)\in\mathbb{Z}$ is the number of contact points at time $t$.
The contact points are the estimated positions where the robot's cable touches the obstacles if the cable is fully stretched.
Similar to the turning points described in Section \ref{sec: prelim}, contact points form a shorter homotopy-equivalent cable configuration.
The procedure of updating the representation and the determination of contact points will be detailed in Section \ref{sec: homotopy}.
\begin{figure}[!t]
\centering
\includegraphics[width=\linewidth]{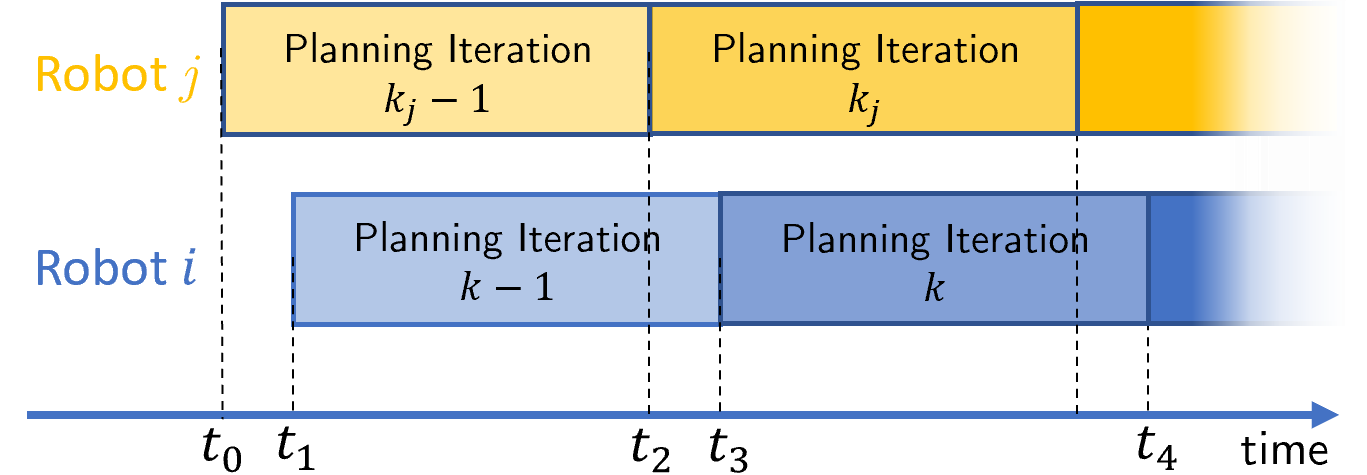}
\caption{\footnotesize 
An illustration of decentralized and asynchronous planning using a common clock time. From time $t_0$ to $t_2$, robot $j$ computes its trajectory for iteration $k_j-1$ (we use the subscript $j$ in $k_j$ to differentiate the iteration number of robot $j$ from that of robot $i$, as they do not share the same number of planning iterations). At time $t_2$, robot $j$ finishes computing its trajectory and publishes it to the communication network. {\cb Between $t_2$ and $t_3$}, robot $i$ receives robot $j$’s updated trajectory ({\color{black} the times of publishing and receiving are different due to a communication delay}), then it is able to compute an updated trajectory in iteration $k$ that avoids the future trajectories of robot $j$.
}
\label{fig:asyplanning}
\end{figure}

The trajectory handler stores the future trajectory of the robot, in the form of $\eta_i$ pieces of polynomial curves, each expressed as polynomial coefficients $\aug{\mathbf{E}}_l\in\mathbb{R}^{(\trajorder+1)\times3}$ for a 3-dimensional $\trajorder$-th order polynomial.
Hence, the future position of the robot at time $t$ can be predicted as
\begin{align}
    \aug{\vb{p}}(t)=\aug{\mathbf{E}}_l^\top \vb{g}(t-t_{i,l}), \,\forall t\in[t_{i,l},t_{i,l+1}],\,l=0\dots\eta_i-1,\nonumber
\end{align}
where $t_{i,0}$ is the current time, $[t_{i,l},t_{i,l+1}]$ is the valid period of the trajectory.
{\color{black}Each robot plans its trajectories iteratively and online,}
i.e., new trajectories are generated while the robot is moving toward the goal.
Each planning iteration is asynchronous with the planning iterations of other robots, as illustrated in Figure \ref{fig:asyplanning}.
At every planning iteration, the robot computes a trajectory starting at time $\init{t}$, which is ahead of the robot's current time by the estimated computation time for the trajectory.
The input to the planning module includes the updated topological status from the Homotopy Update module, the initial states of the robot, $\init{\aug{\vb{p}}}$, $\init{\aug{\vb{v}}}$, $\init{\aug{\vb{a}}}$ at time $\init{t}$, 
and the future trajectories of other robots.
The planning module includes a front-end trajectory finder and a back-end trajectory optimizer, which are described in detail in Sections \ref{sec: search} and \ref{sec: opt}, respectively. 
The output trajectory will replace the existing trajectory in the trajectory handler starting from time $\init{t}$.

At every iteration, robot $i$ broadcasts the following information to the network: (1) its current position $\aug{\vb{p}}_{i}(t)$,
(2) its future trajectory in the form of polynomial coefficients, 
and (3) its current list of contact points
$\vb{c}_{i,1}(t),\dots,\vb{c}_{i,\numcon_i(t)}(t)$.
This message will be received by all other robots $j\in\myset{I}_n\backslash i$.
 The published trajectory is annotated with a common clock known to all robots. Hence, other robots are able to take into account both the spatial and temporal profile of the trajectory in their subsequent planning.

In the following sections, it is considered that all computations and trajectory planning are conducted from the perspective of a robot $i$.
All other robots are called the collaborating robots of robot $i$.
When no ambiguity arises, we omit the subscript $i$ in many expressions for simplicity.
\section{Multi-robot Homotopy Representation}\label{sec: homotopy}
We present a multi-robot tether-aware representation of homotopy constructed from the positions of the robots as well as their cable configurations.
The procedure for updating the homotopy representation in every iteration is detailed in Algorithm \ref{alg: homotopyupdate}, 
which includes updating the word (Section \ref{subsec: lineseg}), reducing the word (Section \ref{subsec: reduce}) and updating the contact points (Section \ref{subsec: contact}).

\begin{algorithm}
\DontPrintSemicolon
\SetKwBlock{Begin}{function}{end function}
\Begin($\text{homotopyUpdate} {(}h{(}k-1{)}, \mathcal{P}{(}k{)}, \mathcal{P}{(}k-1{)}, \{\mathcal{C}_j{(}k{)}\}_{j\in\myset{I}_n\backslash i}, \{\mathcal{C}_j{(}k-1{)}\}_{j\in\myset{I}_n},  \Obset{)}$)
{
$h(k)\leftarrow$UpdateWord$($$h(k-1)$, $\mathcal{P}\left(k-1\right)$, $\mathcal{P}\left(k\right)$, $\{\mathcal{C}_j(k-1)\}_{j\in\myset{I}_n}$, $\{\mathcal{C}_j{(}k{)}\}_{j\in\myset{I}_n\backslash i}$, $\Obset$$)$\;
$h(k)$$ \leftarrow $reduction$($$h(k)$, $\{\mathcal{C}_j{(}k{)}\}_{j\in\myset{I}_n\backslash i}$$)$\;
$\mathcal{C}_i(k)$ $\leftarrow$updateContactPoints$($$\vb{p}(k-1)$, $\vb{p}(k)$, $h(k)$, $\mathcal{C}_i(k-1)$$)$\; 
  \Return{$h(k)$, $\mathcal{C}_i(k)$}
}
\caption{Homotopy update}\label{alg: homotopyupdate}
\end{algorithm}

\subsection{Construction of Virtual Line Segments and 
Updating the Word}\label{subsec: lineseg}
{\cb To compute the representation of homotopy for the planning robot, 
a set of virtual line segments of the collaborating robots and the obstacles are created. Crossing one of these segments indicates an interaction with the corresponding obstacle or robot.}
To obtain the virtual segments of the static obstacles, we first construct $m$ non-intersecting lines $o_j$, $\forall j\in \myset{I}_m$, such that each of them passes through the interior of an obstacle $O_j$.
Each line $o_j$ is separated into three segments by two points on the surface of the obstacle, $\gbf{\zeta}_{j,0}$ and $\gbf{\zeta}_{j,1}$. 
{\cb The virtual segments, labelled as $o_{j,0}$ and $o_{j,1}$, are the two segments outside $O_j$, as shown in Figure \ref{fig:hsig_new}.}
To build the virtual line segments of the collaborating robots, we use the positions and the contact points obtained from the communication network.
The shortened cable configuration of robot $j$, labelled as $r_{j,1}$, is a collection of $\numcon_j+1$ line segments constructed by joining in sequence its base point $\vb{b}_j$, 
its contact points $\vb{c}_{j,1},\vb{c}_{j,2},\dots,\vb{c}_{j,\numcon_j}$ and finally its position $\vb{p}_j$.
The extension of $r_{j,1}$ beyond the robot position $\vb{p}_j$ until the workspace boundary is labelled as $r_{j,0}$.
We call $r_{j,1}$ and $r_{j,0}$ the cable line and extension line of robot $j$, respectively.
Figure \ref{fig:hsig_new} displays the actual cable configuration and the corresponding line segments of robot $j$. 
{\cb By constructing the virtual line segments of all robots and obstacles}, a word of a path is obtained by adding the letter corresponding to the line segment crossed in sequence, regardless of the direction of crossing. 
An example is shown in Figure \ref{fig:hsig_new}, where the word for the black dashed path for robot $i$ is $ o_{1,0}o_{2,0}o_{2,0}o_{2,0}r_{j,1}o_{3,1}$.

\begin{figure*}[!t]
\centering
\includegraphics[width=0.9\linewidth]{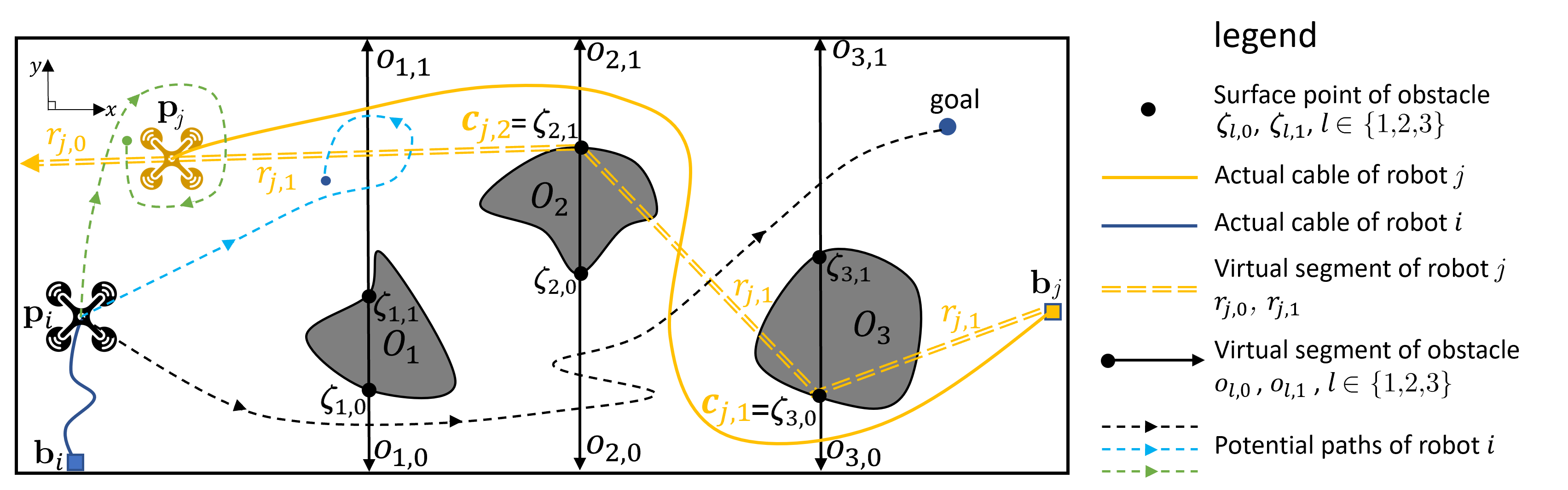}
\caption{\footnotesize {\color{black}Two tethered robots in a workspace consisting of three static obstacles.} The yellow curve represents the cable of robot $j$ and the yellow double dashed line represents the line segments of robot $j$.
In this case, the contact points for robot $j$ are the same as $\vb{\zeta}_{2,1}$ and $\vb{\zeta}_{3,0}$.
{\cb With robot $j$ staying static at $\vb{p}_j$,} the word of the black dashed path for robot $i$ is $ o_{1,0}o_{2,0}o_{2,0}o_{2,0}r_{j,1}o_{3,1}$.}
\label{fig:hsig_new}
\end{figure*}

In an online implementation, the word, denoted as $h(k)$, can be updated iteratively based on the incremental movements of the robots.
To ensure that each robot starts with an empty word $h(0)=$` ', we make the following assumption:
\begin{ass}{(Initial Positions)}\label{ass: initialpos}
The base positions and the initial positions of the robots are placed such that for each robot $i$, its initial simplified cable configuration $\overbar{\vb{p}_i(0)\vb{b}_i}$ does not intersect with any virtual segments $o_{l,f}$ and $r_{j,f}(0)$, $\forall l\in\myset{I}_m$, $j\in\myset{I}_n\backslash i$, $f\in\{0,1\}$, where $r_{j,1}(0)=\overbar{\vb{p}_j(0)\vb{b}_j}$.
\end{ass}

\begin{algorithm}
\DontPrintSemicolon
\SetKwBlock{Begin}{function}{end function}
\Begin($\text{UpdateWord}{(}h{(}k-1{)},\mathcal{P}{(}k-1{)}, \mathcal{P}{(}k{)}, \{\mathcal{C}_j{(}k-1{)}\}_{j\in\myset{I}_n}, \{\mathcal{C}_j{(}k{)}\}_{j\in\myset{I}_n\backslash i}, \Obset{)}$)
{
    $h(k)\leftarrow h(k-1)$\;
    $r_{j,f}(k),r_{j,f}(k-1) \leftarrow \text{getVirtualSegments}()$\;
    \uIf{$\overbar{\vb{p}_i(k-1)\vb{p}_i(k)}$ crosses $\textornot{o}_{j,f}$, $j\in \myset{I}_m, f\in\{0,1\}$, \label{ln: homo_normal_start}}
    {
        Append ${o}_{j,f}$ to $h(k)$
    }
    \uIf{$r_{j,f}$ sweeps across $\vb{p}_i$ {\color{black}within the time interval $\left(k-1,k\right]$}, $j\in \myset{I}_n\backslash i, f\in\{0,1\}$,}
    {
        Append ${r}_{j,f}$ to $h(k)$
    }\label{ln: homo_normal_end}
    \uIf{$r_{j,0}$ sweeps across $\vb{b}_i$ {\color{black}within the time interval $\left(k-1,k\right]$}, $j\in \myset{I}_n\backslash i$,\label{ln: homo_spe_start}}
    {
        Append ${r}_{j,0}$ to $h(k)$
    }\label{ln: homo_spe_end}    
  return $h(k)$
  }
  \label{updateword}
\caption{{\color{black}Word updating}}\label{alg: crossing}
\end{algorithm}

{\cb The procedure for updating the word at every discretized time $k$ is shown in Algorithm \ref{alg: crossing}.}
There are two cases where a letter can be appended to the word:
\begin{enumerate}
    \item an active crossing case (line \ref{ln: homo_normal_start}-\ref{ln: homo_normal_end}), where robot $i$ crosses a virtual line segment during its movement from $\vb{p}(k-1)$ to $\vb{p}(k)$, or in the case when a collaborating robot $j$ is moving, robot $i$ is swept across by a robot $j$'s virtual segment $r_{j,f}$, $f\in\{0,1\}$, as shown in Figure \ref{fig: crossing}; 
    \item a passive crossing case (line \ref{ln: homo_spe_start}-\ref{ln: homo_spe_end}), where the extension line of a collaborating robot, $r_{j,0}$, $j\in \myset{I}_n\backslash i$, sweeps across robot $i$'s base, $\vb{b}_i$. 
\end{enumerate}
The second case is needed to ensure consistency of representation regarding different starting positions of robots. 
As shown in Figure \ref{fig: crossing}, both positions $\vb{p}_j(k-2)$ and $\vb{p}_j(k)$ are valid starting positions of robot $j$, i.e. they can be chosen as $\vb{p}_j(0)$ because they satisfy Assumption \ref{ass: initialpos} with $\vb{p}_i(k)$ chosen as the initial position of robot $i$. Therefore, both positions should induce an empty word for robot $i$.
{\cb This can be ensured only by appending the letter $r_{j,0}$ twice when robot $j$ moves from $\vb{p}_j(k-2)$ to $\vb{p}_j(k)$,} one time when $r_{j,0}$ sweeps across robot $i$ and the other time when $r_{j,0}$ sweeps across $\vb{b}_i$.
Then a reduction procedure cancels the consecutive same letters (Section \ref{subsec: reduce}), resulting in an empty word.

\begin{figure}[!t]
\centering
\includegraphics[width=1.0\linewidth]{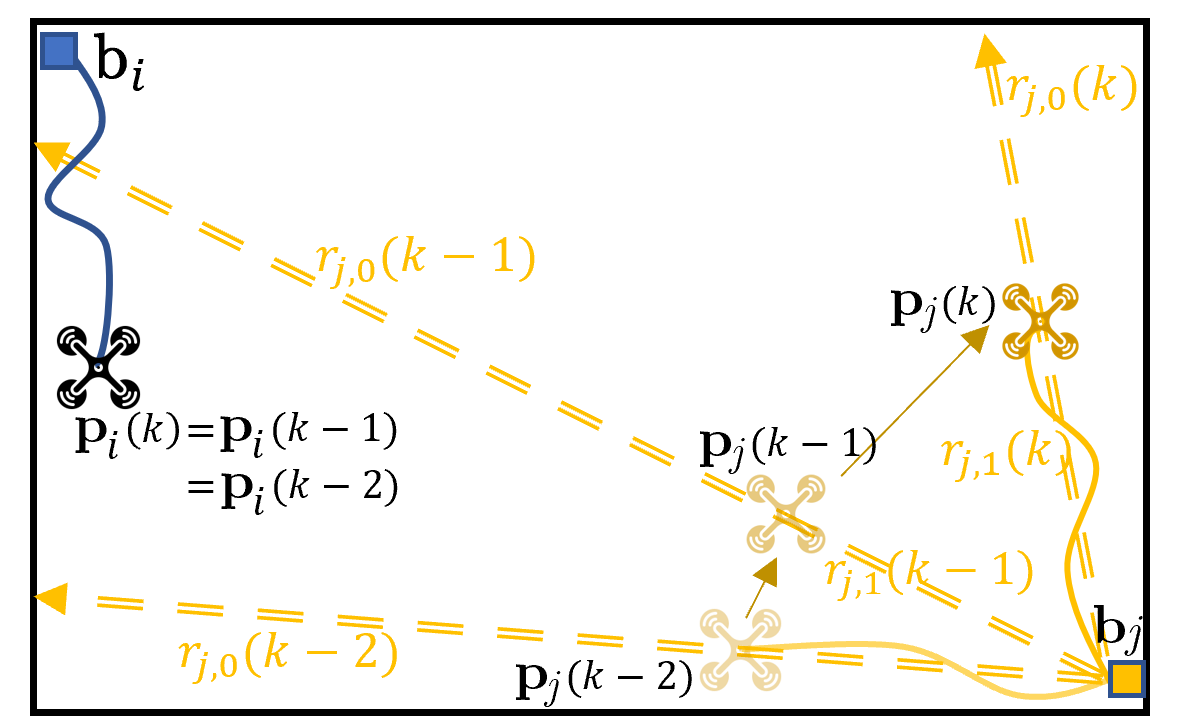}
\caption{\footnotesize {\color{black}An illustration of robot $i$ crossing the extension line of robot $j$. Between time $k-2$ and $k-1$, robot $i$ crossed the segment $r_{j,0}$, because $\vb{p}_i(k-2)$ and $\vb{p}_i(k-1)$ lied on different sides of $r_{j,0}(k-2)$ and $r_{j,0}(k-1)$, respectively.
It can be also viewed as robot $i$ swept across by $r_{j,0}$.
Therefore, $r_{j,0}$ should be appended to $h(k-1)$.
Between time $k-1$ and $k$, robot $i$'s base was swept across by $r_{j,0}$,} hence $r_{j,0}$ should be appended to $h(k)$.
}
\label{fig: crossing}
\end{figure}

\begin{rem}
In practice, Assumption \ref{ass: initialpos} can be always satisfied by adjusting the initial positions and the base positions of the robots, such that any intersections with the virtual segments are avoided.
For a given set of initial positions and base positions, valid virtual segments $o_j$ can be constructed by selecting a reference point from the interior of each obstacle (the reference points should not coincide), and then sampling (uniformly or randomly) directions for extending the points into sets of parallel lines. The set of lines satisfying Assumption \ref{ass: initialpos} is chosen. This procedure only needs to run at the initialization stage and $o_{j,0}$ and $o_{j,1}$ can be saved for online use.
\end{rem}

\subsection{Reduced Form of Homotopy Representation}\label{subsec: reduce}
{\cb A critical step for establishing a homotopy invariant for a topological space is to identify a set of words that are topologically equivalent to the empty word and remove them from the original hotomotpy representation. These words correspond to the elements of the fundamental group mapped to the identity element \cite{bhattacharya2018path}. }
For example, in \cite{Allen2002Algebraic}, {\cb the same consecutive letters,} corresponding to crossing and then uncrossing the same segment, can be removed from the original word to obtain a reduced word.
{\cb In our approach, we identify and remove not only the identical consecutive letters but also combinations of letters representing paths or sub-paths that loop around the intersections of virtual segments.}
For example, in Figure \ref{fig:hsig_new}, the dashed blue path that loops around the intersection of $r_{j,1}$ and $o_{1,1}$ can be topologically contracted to a point (it is null-homotopic).
Therefore, its word $o_{1,1}r_{j,1}o_{1,1}r_{j,1}$ can be replaced by an empty string (for further understanding of the topological meaning of null-homotopic loops, readers may refer to Appendix \ref{apd: homotopy3d}).
The reduction procedure is shown in Algorithm \ref{alg: reduction}, {\color{black}which consists of the following steps}:
\begin{enumerate}
    \item Given an unreduced homotopy representation $h(k)$ at time $k$,
we check whether it contains a pair of identical letters $\chi_{j,f}$,
$\chi\in\{o, r\}$, $f\in\{0,1\}$, and extract the string $\chi_{j,f}\chi^1\chi^2\dots\chi^\omega\chi_{j,f}$ from $h(k)$. {\cb Here $\omega$ is the number of letters between the pair.} 
\item We check if the virtual segment $\chi_{j,f}$ intersects with all of the segments $\chi^1$, $\chi^2$, $\dots$, $\chi^{\omega}$ at time $k$. If so, remove both letters $\chi_{j,f}$ from $h(k)$.
\item {\cb After a removal}, go to step 1) and start checking again.
\end{enumerate}

\begin{algorithm}
\DontPrintSemicolon
\SetKwBlock{Begin}{function}{end function}
\Begin($\text{Reduction}{(}h{(}k{)},\{\mathcal{C}_j{(}k{)}\}_{j\in\myset{I}_n\backslash i}{)}$)
{
Reduced $\leftarrow$ True\;
  \While{Reduced}
    {
    Reduced $\leftarrow$ False\;
    \For{$\alpha\in[1,\text{size}(h(k))-1]\cap\mathbb{Z}$}
    {
        \For{$j\in[\alpha+1,\text{size}(h(k))]\cap\mathbb{Z}$}
        {
        \uIf{\cb$\text{entry}(h(k),\alpha)==\text{entry}(h(k),j)$}
        {
            {\cb remove $\text{entry}(h(k),\alpha)$ and $\text{entry}(h(k),j)$}\;
            Reduced $\leftarrow$ True\;
            break
        }
        \uElseIf{\cb$\text{entry}(h(k),\alpha)$ does not cross $\text{entry}(h(k),j)$}
        {break}
        }
        \uIf{Reduced}{break}
    }
    }

  return $h(k)$
  }
  \label{reduction}
%   \Return{True}
\caption{{\color{black}Word reduction}}\label{alg: reduction}
\end{algorithm}
The reason behind the procedure is that, in a static environment, the string $\chi_{j,f}\chi^1\chi^2\dots\chi^\omega\chi_{j,f}$ constitutes a part of the loop if the segment $\chi_{j,f}$ intersects with all of the segments in between, $\chi^1$, $\dots$,  $\chi^\omega$.
The loop is in the form of $\chi_{j,f}\chi^1\chi^2\dots\chi^\omega\chi_{j,f}\diamond\Lambda(\chi^1\chi^2\dots\chi^\omega)$, where $\Lambda(\cdot)$ denotes a particular sequence of the input letters.
Figure \ref{fig: loop} shows two particular forms of the loop for $\omega=3$.
{\cb As a loop is an identity element in the fundamental group,} we can write $\chi_{j,f}\chi^1\chi^2\dots\chi^\omega\chi_{j,f} = \Lambda(\chi^1\chi^2\dots\chi^\omega)^{-1}$.
Assuming that $\Lambda(\chi^1\dots\chi^\omega)$ takes the form shown in Figure \ref{fig: loopa}, i.e.,  $\Lambda(\chi^1\dots\chi^\omega) = \chi^\omega\chi^{\omega-1}\dots\chi^1$, we can establish a relation $\chi_{j,f}\chi^1\chi^2\dots\chi^\omega\chi_{j,f} = \chi^1\chi^2\dots\chi^\omega$, {\cb thus removing both letters $\chi_{j,f}$}.
{\cb As a special case,} although virtual segments $r_{j,0}$ and $r_{j,1}$ share a common point at robot $j$'s position, {\cb they are considered non-intersecting. Otherwise, any loop around the robot (e.g., the dashed green path in Figure \ref{fig:hsig_new}) will be reduced to a single letter which reflects a different topological meaning.}
\begin{figure}
	\centering
    \subcaptionbox{ \footnotesize \label{fig: loopa}}[0.45\linewidth]{\includegraphics[height=2.5cm]{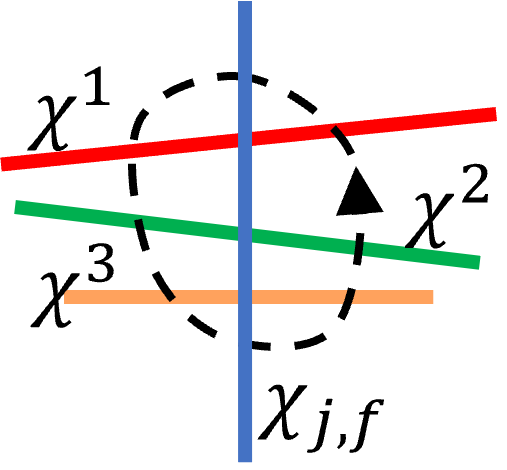}}
    \subcaptionbox{\footnotesize \label{fig: loopb}}[0.45\linewidth]{\includegraphics[height=2.5cm]{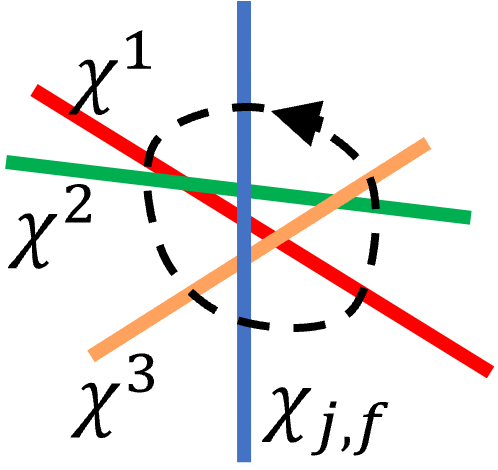}}
	\caption{\footnotesize Two loops around the intersections among 4 segments. \\{\color{black}(a) Loop expressed as $\chi_{j,f}\chi^1\chi^2\chi^3\chi_{j,f}\chi^3\chi^2\chi^1$. (b) Loop expressed as $\chi_{j,f}\chi^1\chi^2\chi^3\chi_{j,f}\chi^1\chi^2\chi^3$.}
    }  \label{fig: loop}
\end{figure}

It should be noted that the reduction rules introduced here are insufficient to generate a true homotopy invariant for the topological space considered in our scenario, which requires a much more involved construction in a 4-dimensional (X-Y-Z-Time) space and is beyond the scope of this paper.
However, the reduced homotopy representation captures important topological information about the robot to check for the entanglement of cables.
The condition to identify entanglement can be stated as a two-entry rule:
a robot is considered to be risking entanglement at time $k$, if its reduced homotopy representation contains two or more letters corresponding to the same collaborating robot, i.e. $h(k)$ consists of two or more entries of $r_{j,f}$, $j\in\myset{I}_n\backslash i$, $f\in\{0,1\}$.
{\cb This rule can be justified as follows:}
\begin{prop}\label{prop: twoentries}
 Consider two robots moving in a 2-D workspace $\myset{W}$ without any static obstacles, each tethered to a fixed base at $\vb{b}_1$ and $\vb{b}_2$. The starting positions of the robots satisfy Assumption \ref{ass: initialpos}.
 An entanglement occurs only if the reduced homotopy representation of at least one of the robots, $h_i(k)$, $i\in\myset{I}_2$, contains at least two entries related to the other robot, $ \textornot{r}_{j,f}$, $j\in\myset{I}_2\backslash i$, $f\in\{0,1\}$.
\end{prop}
\begin{proof}
Figure \ref{fig: entangle1g} illustrates the simplest entanglement scenario between the two robots. {\cb More complex 2-robot entanglements are developed from this scenario when the robots circle around each other, resulting in more entries in their homotopy representations.}
Hence, it is sufficient to evaluate the simplest case.
A sequence of crossing actions is required to realize this case. Robot $i$ crosses the cable line of robot $j$, $r_{j,1}$, followed by robot $j$ crossing the cable line of robot $i$, $r_{i,1}$, $i\in\myset{I}_2$, $j\in\myset{I}_2\backslash i$.
Due to the requirement on the initial positions (Assumption \ref{ass: initialpos}), {\cb before crossing the cable of robot $j$ by robot $i$, the extension line $r_{i,0}$ will sweep across either robot $j$ or base $j$,} causing the reduced word of robot $j$ to be $h_j={r}_{i,0}$.
{\cb After crossing $r_{i,1}$ by robot $j$,} the word of robot $j$ is guaranteed to have two entries, $h_j = r_{i,0}r_{i,1}$.
\end{proof}

\begin{figure*}
	\centering
    \subcaptionbox{ \label{fig: ent_pro1}}[0.24\linewidth]{\includegraphics[width=\linewidth]{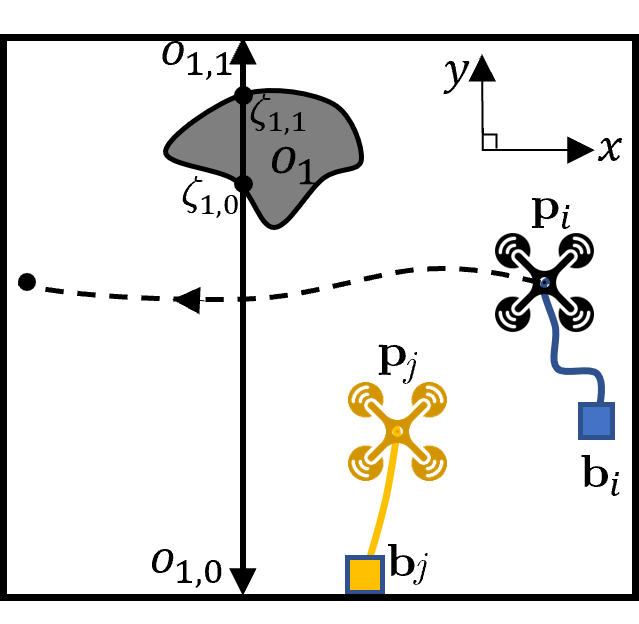}}
    \subcaptionbox{ \label{fig: ent_pro2}}[0.24\linewidth]{\includegraphics[width=\linewidth]{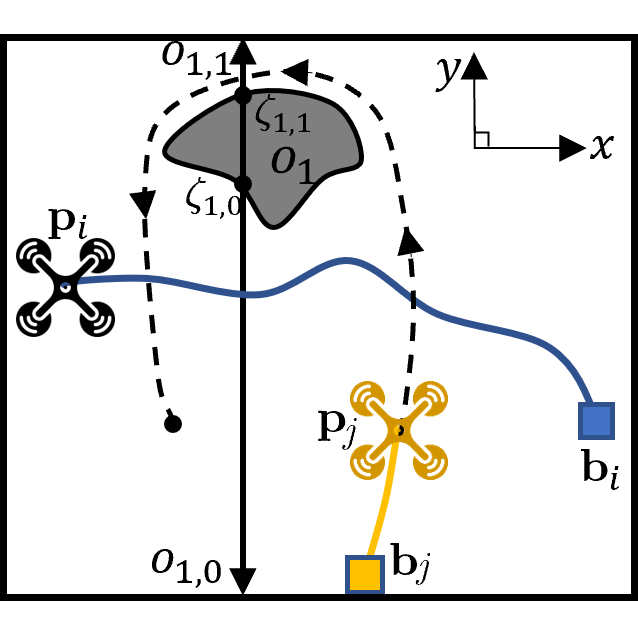}}
    \subcaptionbox{ \label{fig: ent_pro3}}[0.24\linewidth]{\includegraphics[width=\linewidth]{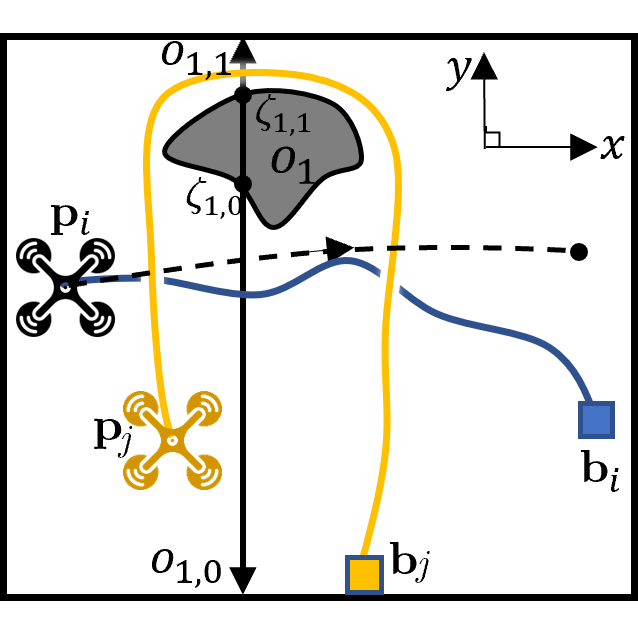}}
    \subcaptionbox{ \label{fig: ent_pro4}}[0.24\linewidth]{\includegraphics[width=\linewidth]{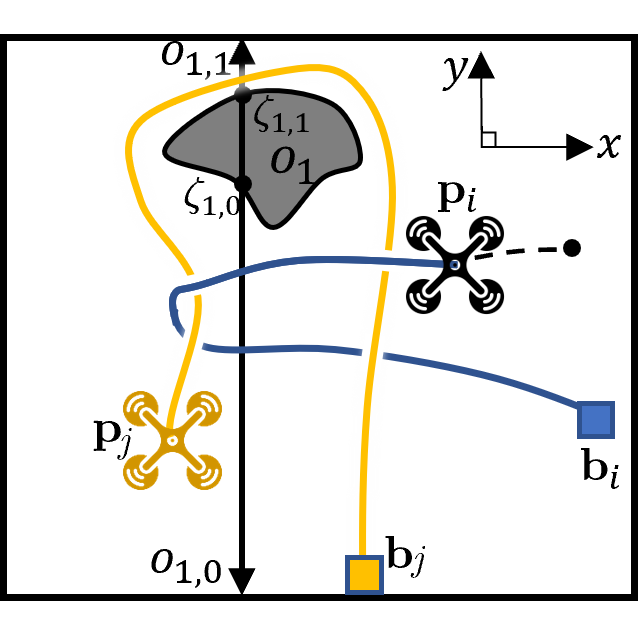}}    
	\caption{ \footnotesize Two UAVs follow a series of movements, resulting in an entanglement situation. {\cb The black dashed paths indicate the paths followed by each robot till their goals.} The blue and yellow curves are the cables of robot $i$ and robot $j$, respectively. (a) Robot $i$ moves to the negative $x$ direction. (b) {\cb Robot $j$ follows the indicated path crossing the cable of robot $i$ two times and passing an obstacle.} (c) Robot $i$ moves in the positive $x$ direction. (d) Robot $i$'s movement is restricted because its cable is tangled with robot $j$'s cable.}  \label{fig: ent2}
\end{figure*}

\begin{figure}[!t]
\centering
\includegraphics[width=0.6\linewidth]{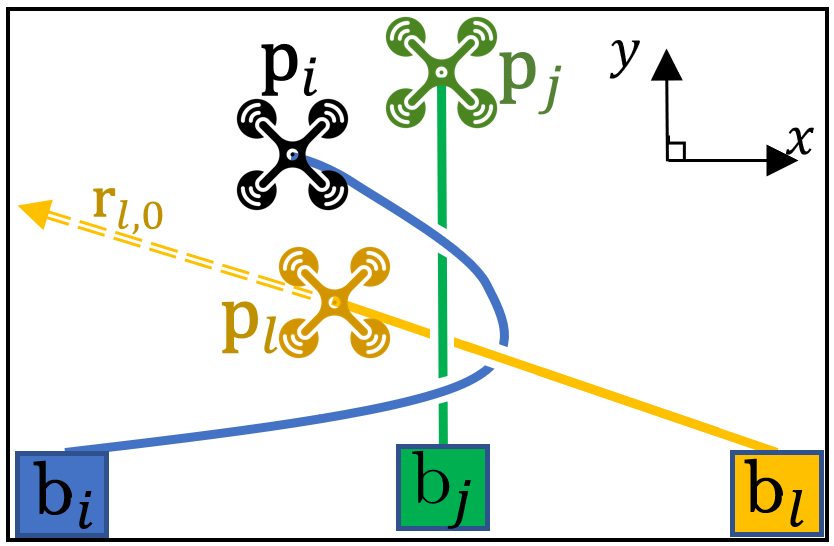}
\caption{\footnotesize A 3-robot entanglement scenario. The solid lines and curves show the cable configurations with their Z-ordering. The double dashed line is a virtual segment.
}
\label{fig: 3-rob ent}
\end{figure}
For operations using $3$ and more robots, any pair-wise entanglements involving only two robots can be detected using the same argument as Proposition \ref{prop: twoentries}. 
Figure \ref{fig: 3-rob ent} illustrates a more complicated $3$-robot entanglement, where robot $i$'s motion is hindered due to the cables of both robot $j$ and $l$ (removing either robot $j$ or $l$ releases robot $i$ from the entanglement).
{\cb This scenario is realized by the sequence of movements where robot $i$ crosses (1) the cable line of robot $j$, (2) the extension line $r_{l,0}$ of robot $l$, (3) the cable line of robot $j$ again.}
The homotopy representation of robot $i$ is $h_i=r_{j,1}r_{l,0}r_{j,1}$. {\cb Since segments $r_{l,0}$ and $r_{j,1}$ do not intersect, $h_i$ is nonreducible.} Therefore, the entanglement can be detected using the two-entry rule.

{\color{black}Considering obstacle-ridden environments, Figure \ref{fig: ent2} illustrates an entanglement caused by two robots following the intended paths sequentially.
Since the word of robot $j$, $h_j=r_{i,1}o_{1,1}r_{i,1}$, is nonreducible,  
the entanglement can also be identified using the two-entry rule. 
Intuitively, the two-entry rule can be used to identify instances when a robot partially circles around another robot (Figure \ref{fig: entangle1g}), and instances when a robot circles around a part of the cable of another robot in a topologically non-trivial way (Figure \ref{fig: ent_pro3} and Figure \ref{fig: 3-rob ent}).}

\subsection{Determination of Contact Points}\label{subsec: contact}

The list of contact points is updated at every iteration, {\cb along with updating and reducing the word.}
Intuitively, a contact point addition (or removal) occurs when a new bend of the cable is created at (or an existing bend is released from) {\cb the surface of the obstacles passed by the robot.}
For efficiency, we adopt a simplification of the obstacle shapes, such that each obstacle $j$ is only represented as a thin barrier $\overbar{\gbf{\zeta_{j,0}}\gbf{\zeta_{j,1}}}$.
The detailed procedure is described in Algorithm \ref{alg: funnel}, 
in which the indication of the timestamp $(k)$ in the expressions $\vb{c}_1(k)$, \dots, $\vb{c}_{\numcon_i(k)}(k)$ is omitted for brevity.
This algorithm checks if the robot has crossed any lines linking a contact point to the surface points of the obstacles, or any lines linking the two consecutive contact points; if such crossing happens, {\cb contact points are added or removed accordingly.}
{\cb Figure \ref{fig: funnel} illustrates Algorithm \ref{alg: funnel}.}
\begin{algorithm}
\DontPrintSemicolon
\SetKwBlock{Begin}{function}{end function}
\Begin($\text{updateContactPoints}{(}\vb{p}{(}k-1{)}, \vb{p}{(}k{)}, h{(}k{)},  \mathcal{C}_i{(}k-1{)}{)} $)
{
    $\mathcal{C}_i(k)\leftarrow \mathcal{C}_i(k-1)$\;
  \While{$\numcon_i\geq1$ and \label{ln:contact1}\\$\overline{\vb{p}(k-1) \vb{p}(k)}$ intersects with $ \overleftrightarrow{\vb{c}_{\numcon_i-1} \vb{c}_{\numcon_i}}$}
  {
        \tcp{$\vb{c}_0 = \vb{b}$ in this algorithm}
        remove $\vb{c}_{\numcon_i}$ from $\mathcal{C}_i(k)$\;
        $\numcon_i\leftarrow \numcon_i-1$\label{ln:contact2}
    }

    \ForAll{\cb $l\in [\text{index}(\vb{c}_{\numcon_i})+1,\text{size}(h(k))]\cap\mathbb{Z}$\label{ln:contact3}}
    {
    \uIf{$\exists j\in \myset{I}_m, f\in\{0,1\}$, s.t. {\cb\text{entry}$(h(k),l)==\textornot{o}_{j,f}$,}}
    {
        \uIf{$\overline{\vb{p}(k-1) \vb{p}(k)}$ intersects with $\overleftrightarrow{\vb{c}_{\numcon_i} \gbf{\zeta}_{j,f}}$}
        {
        $\numcon_i\leftarrow \numcon_i+1$\;     
        add $\vb{c}_{\numcon_i}\leftarrow\gbf{\zeta}_{j,f}$ to $\mathcal{C}_i(k)$\label{ln:contact4}\;
        }    
    }

    }
    
  return $\mathcal{C}_i(k)$
  }\label{funn}
%   \Return{True}
\caption{Contact point update}\label{alg: funnel}
\end{algorithm}
\begin{figure*}
	\centering
    \subcaptionbox{ \footnotesize time $k$ \label{fig: fun1a}}[0.29\linewidth]{\includegraphics[height=5.5cm]{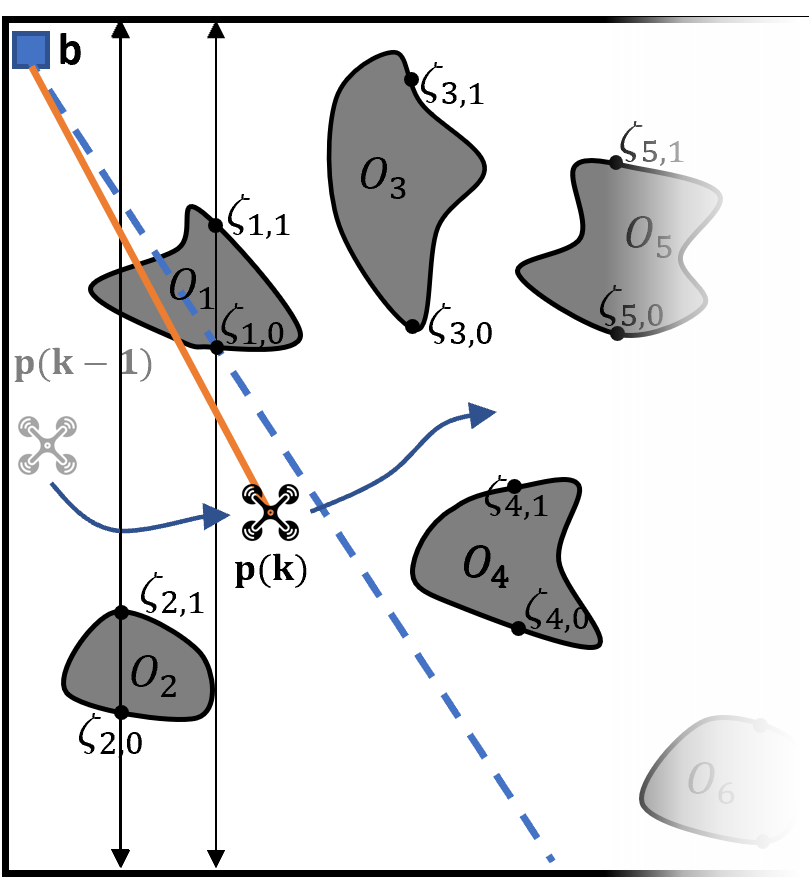}}
    \subcaptionbox{\footnotesize time $k+1$  \label{fig: fun1b}}[0.33\linewidth]{\includegraphics[height=5.5cm]{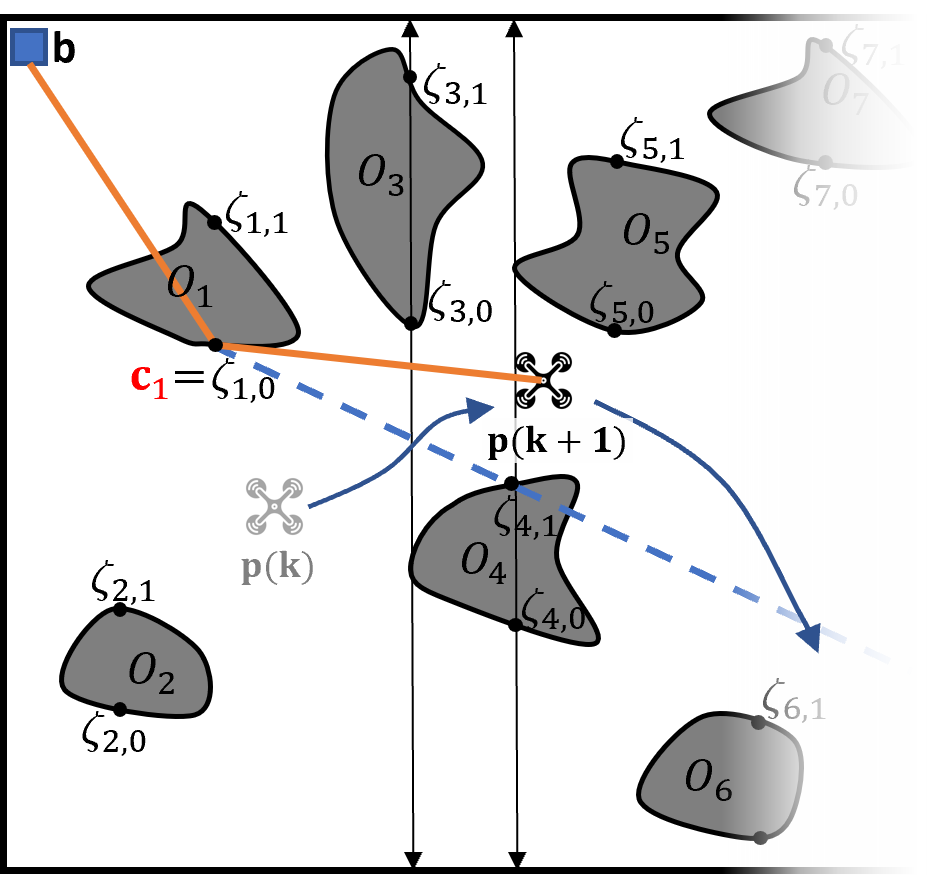}}
    \subcaptionbox{\footnotesize time $k+2$  \label{fig: fun1c}}[0.36\linewidth]{\includegraphics[height=5.5cm]{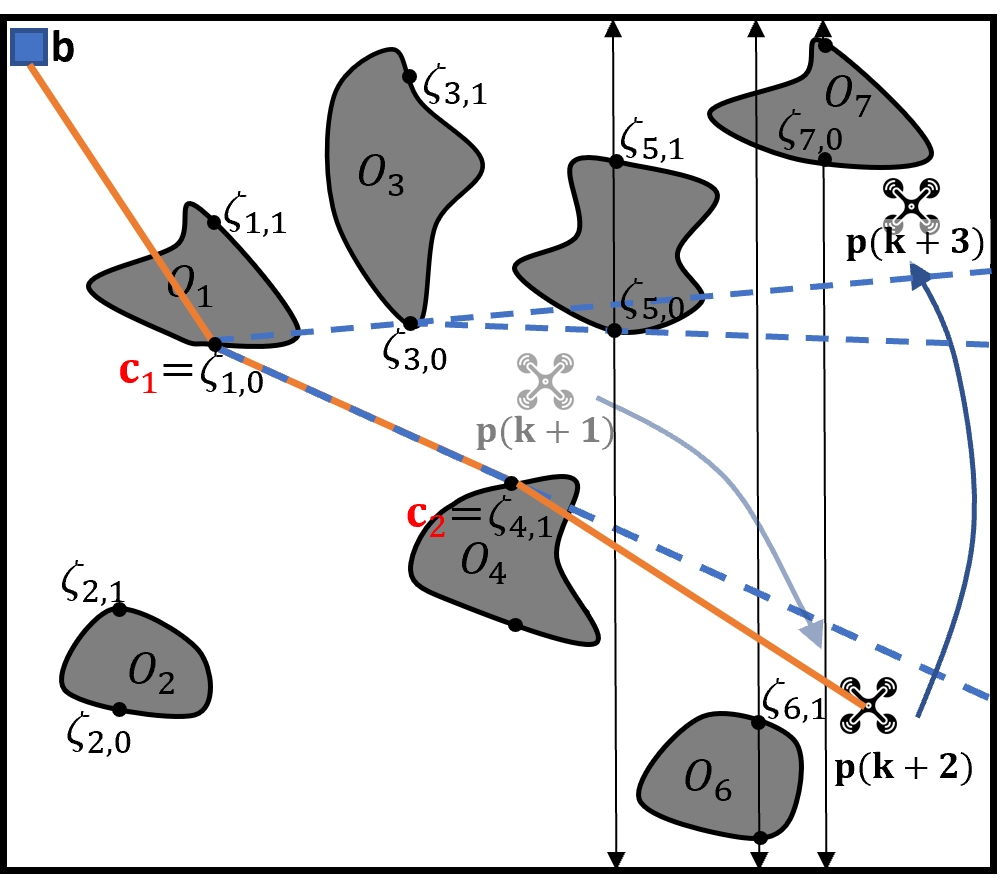}}
	\caption{\footnotesize A sequence of movements of robot $i$. 
	Figures (a) and (b) only show a part of the workspace.
	Orange line segments are the shortened cable configuration at each time. The blue dashed lines are the lines intersected by the robots' movements, resulting in addition or removal of contact points. 
	(a) At time $k$, robot has no contact point.
	$h(k)=\textornot{o}_{2,1}\textornot{o}_{1,0}$.
	(b) In the interval between $k$ and $k+1$, the robot crosses the line linking $\vb{b}_i$ and $\gbf{\zeta}_{1,0}$, causing $\gbf{\zeta}_{1,0}$ to be a contact point.
	$h(k+1) = \textornot{o}_{2,1}\textornot{o}_{1,0}\textornot{o}_{3,0}\textornot{o}_{4,1}$.
	(c) In the interval between $k+1$ and $k+2$, the robot crosses the line linking $\vb{c}_1$ and $\gbf{\zeta}_{4,1}$, causing $\gbf{\zeta}_{4,1}$ to become the second contact point.
	$h(k+2) = \textornot{o}_{2,1}\textornot{o}_{1,0}\textornot{o}_{3,0}\textornot{o}_{4,1}\textornot{o}_{5,0}\textornot{o}_{6,1}\textornot{o}_{7,0}$.
    If the robot moves from $\vb{p}(k+2)$ to $\vb{p}(k+3)$ in one time step, $\gbf{\zeta}_{4,1}$ will be removed from the contact points, {\cb and $\gbf{\zeta}_{3,0}$ and $\gbf{\zeta}_{5,0}$ will be added to the contact points. 
    However, $\gbf{\zeta}_{3,0}$ is not a valid contact point, because no contact between the cable and ${O}_3$ should occur when the cable forms a straight line between $\gbf{\zeta}_{1,0}$ and $\gbf{\zeta}_{5,0}$.}
    }  \label{fig: funnel}
\end{figure*}
We make the following assumption to ensure the correct performance of the algorithm:

\begin{ass}\label{ass: cross_one}

Consider $\mathcal{S}_{j,f}$ as the set of lines linking the obstacle surface point $\gbf{\zeta}_{j,f}$ to the surface points of different obstacles and the base, {\cb $\mathcal{S}_{j,f}=\big\{\overleftrightarrow{\gbf{\zeta}_{j,f}\gbf{\zeta}}|\gbf{\zeta}\in\vb{b}_i\cup\{\gbf{\zeta}_{l,\alpha}\}_{l\in\myset{I}_m\backslash j, \alpha\in\{0,1\}}\big\}$.
Within the time interval $\left(k-1,k\right]$,} only one distinct line in each $\mathcal{S}_{j,f}$ can be crossed by robot $i$, i.e., multiple lines in $\mathcal{S}_{j,f}$ can be crossed by the robot only if they are coincident, $\forall j\in\myset{I}_m$, $f\in\{0,1\}$. 
\end{ass}
{\cb Assumption \ref{ass: cross_one} 
prevents incorrect identification of surface points as contact points. 
In Figure \ref{fig: fun1c}, the robot moving from $\vb{p}(k+2)$ to $\vb{p}(k+3)$ will result in
$\gbf{\zeta}_{3,0}$ added as a contact point, but no contact between the cable and $O_3$ is possible in this case.}
Under Assumption \ref{ass: cross_one}, the only case where multiple contact points should be added or removed in one iteration is when multiple coincident lines are crossed.
Such a case can be handled by the iterative intersection check in Algorithm \ref{alg: funnel} following the sequence of the obstacles added to $h(k)$ (line \ref{ln:contact3} to \ref{ln:contact4}), and the reverse sequence of the contact points added to $\mathcal{C}_i(k)$ (line \ref{ln:contact1} to \ref{ln:contact2}).
In practice, {\cb Assumption \ref{ass: cross_one} holds true if the displacement of the robot is sufficiently small between consecutive iterations, 
which can be achieved by a sufficiently high iteration rate.}

We have the following statement on the property of the shortened cable configuration obtained using Algorithm \ref{alg: funnel}:
\begin{prop}\label{prop: shortest}
 {\cb Consider a robot $i$ tethered to a base $\vb{b}_i$ and moving in a 2-D environment consisting of only thin barrier obstacles }$\overbar{\gbf{\zeta}_{j,0}\gbf{\zeta}_{j,1}}$, $\forall j\in\myset{I}_m$.
 Let Assumptions \ref{ass: initialpos} and \ref{ass: cross_one} hold.
At a time $k$, the consecutive line segments formed by linking $\vb{b}_i$, $\vb{c}_{i,1}(k)$, \dots, $\vb{c}_{i,\numcon_i(k)}(k)$ and $\vb{p}_i(k)$ 
represent the shortest cable configuration of the robot homotopic to the actual cable configuration.
\end{prop}
\begin{proof}[Proof]
 See Appendix \ref{apd: proofshortest}.
\end{proof}

The length of the shortened cable configuration is a lower bound of the actual length of the shortest cable due to the use of simplified obstacle shapes.
In Section \ref{sec: search}, this length is compensated with an extra distance to the surface of the obstacles, to better approximate the required cable length.
As will be shown in Section \ref{sec: simulation}, the use of Algorithm \ref{alg: funnel} enables more efficient feasibility checking than the expensive curve shortening algorithm used in \cite{kim2014path, kim2015path}.
\begin{rem}

Algorithm \ref{alg: funnel} is inspired by the classic funnel algorithm \cite{HERSHBERGER199463}, which is widely used to find shortest homotopic paths in triangulated polygonal regions \cite{Teshnizi2021}. 
{\color{black}Compared to the funnel algorithm, Algorithm \ref{alg: funnel} is more computationally efficient. The reason is that the simplified obstacle shapes with fewer vertices and a maintained list of crossed obstacles in $h(k)$ reduce the number of needed crossing checks. Furthermore, using Algorithm \ref{alg: funnel}, the shortest path can be obtained trivially by linking all contact points. On the other hand, an additional procedure is required in the funnel algorithm to determine the shortest path (because the apex of the funnel may not be the last contact point).
Algorithm \ref{alg: funnel} also provides memory saving compared to the funnel algorithm. The latter requires saving a funnel represented as a double-ended queue (deque) consisting of the boundary points of the funnel, which is always greater than or equal to the list of contact points saved in our algorithm.}
As will be described in Section \ref{subsec: entanglement}, the homotopy update procedure is called frequently to incrementally predict and evaluate the homotopy status of the potential trajectories.
Therefore, the frequent updates of the contact points result in significant computational and memory saving compared to the funnel algorithm with full triangulation.
\end{rem}

\section{Kinodynamic Trajectory Finding}\label{sec: search}
The front end uses kinodynamic A* search algorithm \cite{zhou2019robust} to find a trajectory leading to the goal while satisfying the dynamic feasibility, collision avoidance and non-entanglement requirements.
A search-based algorithm is used instead of a sampling-based algorithm, such as RRT, to ensure better consistency of the trajectories generated in different planning iterations.
To reduce the dimension of the problem, the kinodynamic A* algorithm searches for feasible trajectories in the X-Y plane only. 
For UAVs, the computation of Z-axis trajectory is also required and will affect the feasible region in the X-Y plane. Hence, we design a procedure to generate dynamically feasible trajectories in Z-axis only, interested readers may refer to Appendix \ref{subsec: initialz} for details.

The search is conducted in a graph imposed on the discretized 2-D space augmented with the homotopy representation of the robot. {\cb Each node in the graph is a piece of trajectory of fixed duration $T$, which contains the following information:}
(1) the 2-D trajectory coefficients $\mathbf{E}_l\in\mathbb{R}^{(\trajorder+1)\times2}$, where $l$ is the index starting from the first trajectory as $0$; 
(2) the robot's states at the end of the trajectory $\en{\vb{p}}$, $\en{\vb{v}}$, $\en{\vb{a}}$; 
(3) the robot's homotopy-related information at the end of the trajectory, $\en{h}$ and $\en{\mathcal{C}}$; 
(4) the cost from start $g_\text{c}$ and heuristic cost $g_\text{h}$; (5) a pointer to its parent trajectory.
A node is located in the graph by its final position $\en{\vb{p}}$ and its homotopy representation $\en{h}$.
Successive nodes are found by
applying a set of sampled control inputs $\mathcal{U}=\{[u^\text{x},u^\text{y}]^\top|u^\text{x},u^\text{y}\in\{-\bar{u}, -\frac{\sigma_\text{u}-1}{\sigma_\text{u}}\maxop{u}, \dots, \frac{\sigma_\text{u}-1}{\sigma_\text{u}}\maxop{u}, \maxop{u}\}\}$ to the end states of a node for a duration $T$, where 
$\maxop{u}$ is the maximum magnitude of control input.
In our case, the control input is jerk and the degree of trajectory is $3$. Given a control input $\vb{u}\in\mathcal{U}$ applied to a parent node, the coefficients of the successive trajectory can be obtained as
\begin{align}
    \mathbf{E}_l = [\en{\vb{p}}_{l-1}\;\: \en{\vb{v}}_{l-1}\;\: \frac{1}{2}\en{\vb{a}}_{l-1} \;\:\frac{1}{6}\vb{u}]^\top\in\mathbb{R}^{4\times2}.
\end{align}
Then each of the successive nodes is checked for its dynamic feasibility (Section \ref{subsec: feasibility}), collision avoidance (Section \ref{subsec: collision}) and non-entanglement requirements (Section \ref{subsec: entanglement}). 
{\cb The homotopy update is conducted together with the entanglement check. 
Only the new nodes satisfying the constraints will be added to the open list.}
The search process continues until a node is found that ends sufficiently close to the goal without risking entanglement (based on the two-entry rule).
{\color{black}The heuristic cost $g_\text{h}$ is chosen to be the Euclidean distance between $\en{\vb{p}}$ and the goal.}

\subsection{Workspace and Dynamic Feasibility}\label{subsec: feasibility}
We use the recently published MINVO basis \cite{tordesillas2021minvo} to convert the polynomial coefficients to the control points, which form convex hulls that entirely encapsulate the trajectory and its derivatives.
$\mathcal{Q}_l$, $\mathcal{V}_l$ and $\mathcal{A}_l$ represent the set of 2-D position, velocity, and acceleration control points of a candidate node with index $l$, respectively. To ensure feasibility, we check whether the following inequalities are satisfied:
\begin{align}
    &\posctrlpt\in\mathcal{W},\; \forall\posctrlpt\in\mathcal{Q}_l,\label{eq: poscontrolpoint}\\
    &[\minop{v}^\text{x}\;\minop{v}^\text{y}]^\top\leq\vb{v}\leq[\maxop{v}^\text{x}\;\maxop{v}^\text{y}]^\top,\; \forall\vb{v}\in\mathcal{V}_l,\label{eq: velcontrolpoint}\\
    &[\minop{a}^\text{x}\;\minop{a}^\text{y}]^\top\leq\vb{a}\leq[\maxop{a}^\text{x}\;\maxop{a}^\text{y}]^\top,\; \forall\vb{a}\in\mathcal{A}_l.\label{eq: acccontrolpoint}
\end{align}
%{\cb where $\posctrlpt$, $\vb{v}$, and $\vb{a}$ are the position, velocity, and acceleration control points, respectively.}

\begin{figure}[!t]
\centering
\includegraphics[width=0.8\linewidth]{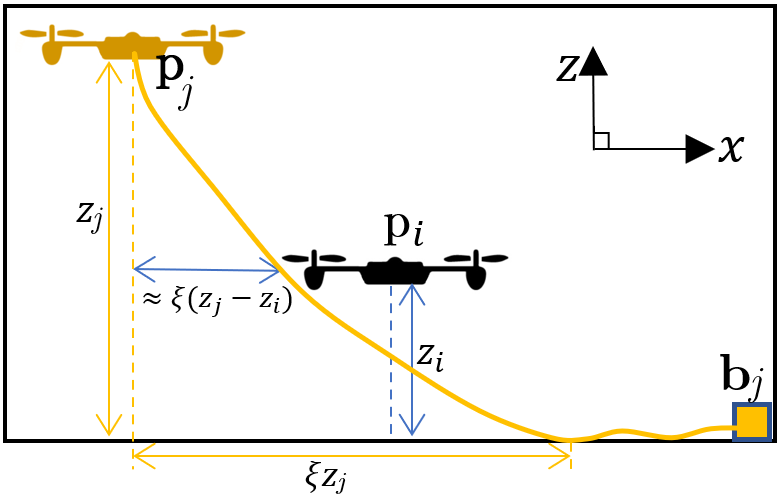}
\caption{\footnotesize UAV $j$'s cable tails behind it, causing an expansion of the collision zone.}
\label{fig: trailing}
\end{figure}

\subsection{Collision Avoidance} \label{subsec: collision}
Due to friction, a UAV flying at a fast speed may have its cable trailing behind, as shown in Figure \ref{fig: trailing}.
We model the range of cable that may stay in the air as $\xi z$ for a UAV flying at altitude $z$. $\xi$ is a constant to be determined empirically.
Therefore, to guarantee non-collision with robot $j$ and its cable, the planning UAV needs to maintain a distance more than $\rho_j+\rho_i+\xi \Delta z_l$, where $\rho_j$ and $\rho_i$ are the radii of UAVs, and $\Delta z_l$ is the maximum possible altitude difference during the planning interval $l$,
which can be obtained from their Z-axis trajectories.
Eventually, we need to check whether the following equality holds:
\begin{align}
    \text{hull}(\mathcal{Q}_{j,l})\oplus\mathbf{B}_{j,l}\cap\text{hull}(\mathcal{Q}_l)=\emptyset,\label{eq: collision avoidance}
\end{align}
where $\text{hull}(\cdot)$ represents the convex hull enclosing the set of control points, $\mathcal{Q}_{j,l}$ is the minimum set of MINVO control points containing the trajectory of robot $j$ during the planning interval $l$, 
$\mathbf{B}_{j,l}$ is the axis-aligned bounding box whose side length is $\rho_j+\rho_i+\xi \Delta z_l$, and $\oplus$ is the Minkowski sum.
We use the Gilbert–Johnson–Keerthi (GJK) distance algorithm \cite{2083} to efficiently detect the collision between two convex hulls.
{\cb Checking the collision with static obstacles can be done similarly by checking the intersections with the inflated convex hulls containing the obstacles.}
\subsection{Cable Length and Non-Entanglement Constraints}\label{subsec: entanglement}
A tethered robot should always operate within the length limit of its cable and never over-stretch its cable.
We check whether a trajectory with index $l$ satisfies this by approximating the required cable length at a position $\vb{p}_i$ using the list of active contact points:
\begin{align}
    \phi_i\geq&\norm{\vb{b}_{i}-\vb{c}_{i,1}} +\sum_{\alpha=1\dots\numcon_i-1}{\norm{\vb{c}_{i,\alpha+1}-\vb{c}_{i,\alpha}}}+\nonumber\\
    &\sum_{\alpha=1\dots\numcon_i}{\mu_\alpha}+\norm{\vb{p}_i-\vb{c}_{i,\numcon_i}}+\maxop{z}_l,\label{eq: cablelength}
\end{align}
where $\maxop{z}_l$ is the maximum altitude of the robot during planning interval $l$ and $\mu_\alpha>0$ is an added length to compensate for the underestimation of cable length using contact points only.
For safety, we can set $\mu_\alpha$ as $2$ times the longest distance from the contact point to any surface points of the same obstacle. 

Algorithm \ref{alg: entangle} incrementally predicts the planning robot's homotopy representation by getting sampled states of the robots (line \ref{ln: sample}) and updating the representation using Algorithm \ref{alg: homotopyupdate} (line \ref{ln: homoupdate}).
After each predicted update, it checks the cable length constraints (line \ref{ln: cablelength}) and the non-entanglement constraint (line \ref{ln: startentanglecheck}-\ref{ln: endentanglecheck}).
{\cb We discard trajectories that fail the entanglement check only due to a crossing action incurred in the current prediction cycle (line \ref{ln: entcriteria}).} 
In this way, we allow the planner starting from an unsafe initial homotopy representation to continue searching and find a trajectory escaping the unsafe situation.
\begin{algorithm}
\DontPrintSemicolon
%\KwIn{$h_i(k-1)$}
%\KwOut{$h_i(k)$, $\Omega(k)$.}
\SetKwBlock{Begin}{function}{end function}
\Begin($\text{EntanglementCheck{(}node{)}} $)
{
 %$l\leftarrow \text{node.index}$\;
  $\mathcal{P}(0)\leftarrow\text{getSampledStates}(E_l,\{\mathcal{P}_{j,l}\}_{j\in \myset{I}_n\backslash i},0)$\;
  \For{$k=1\dots \sigma$}
  {
    $\mathcal{P}(k)\leftarrow\text{getSampledStates}(E_l,\{\mathcal{P}_{j,l}\}_{j\in \myset{I}_n\backslash i},k)\label{ln: sample}$\;
    $\en{h}, \en{\mathcal{C}}\leftarrow\text{homotopyUpdate} (\en{h}, \mathcal{P}(k), \mathcal{P}(k-1),\{\mathcal{C}_j{(}t^\text{in}{)}\}_{j\in\myset{I}_n\backslash i},
    \{\mathcal{C}_j{(}t^\text{in}{)}\}_{j\in\myset{I}_n\backslash i}\cup\en{\mathcal{C}}, \Obset$)\label{ln: homoupdate}\;
    \uIf{$\text{Eqn.(\ref{eq: cablelength}) does NOT hold}$\label{ln: cablelength}}
    {
    return False\;
    }
    \ForAll{$j\in \myset{I}_n\backslash i$\label{ln: startentanglecheck}}
    {
    \uIf{number of entries of $\text{r}_{j,f}$, $f\in\{0,1\}$ in $\en{h}$ increases AND is $\geq2$\label{ln: entcriteria}}
    {
    return False\label{ln: endentanglecheck}
    }
    }
    
    $\mathcal{P}(k-1)\leftarrow\mathcal{P}(k)$
  }
  return True
  }\label{endentangle}
%   \Return{True}
\caption{Cable length and entanglement check}\label{alg: entangle}
\end{algorithm}

\section{Trajectory Optimization}\label{sec: opt}
\begin{figure*}
	\centering
    {\includegraphics[width=0.8\textwidth]{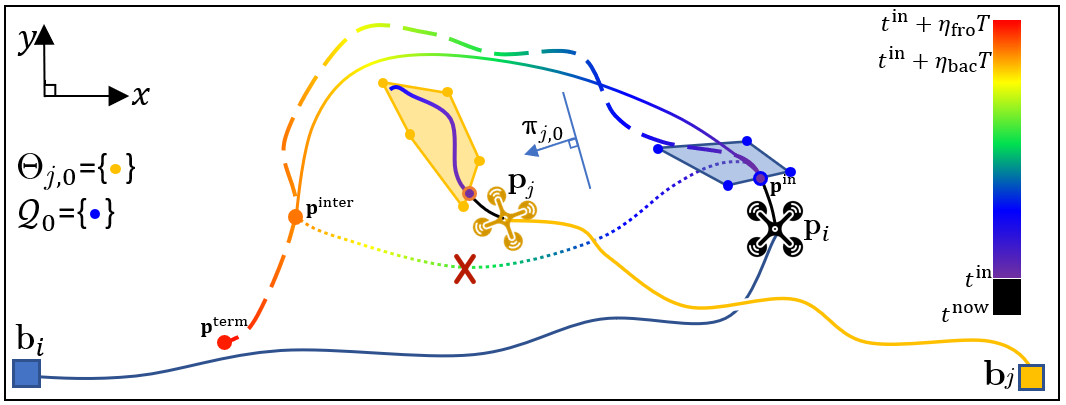}}
	\caption{\footnotesize The dashed rainbow curve is the trajectory obtained from the front-end trajectory finder. The solid rainbow curve is the optimized trajectory. The colors on the rainbow curves indicate the time of the trajectory since $\init{t}$. The dotted rainbow curve is the optimization result if we do not add the non-crossing constraint into the optimization problem, which is an unacceptable shortcut because robot $i$ has already crossed $r_{j,0}$ and should not cross $r_{j,1}$. $\gbf{\pi}_{j,0}$ is the parameter of the line separating robot $i$'s trajectory from robot $j$'s trajectory.}  \label{fig: planning}
\end{figure*}
The output of the front-end planner is $\front{\eta}$ pieces of consecutive polynomial curves, $\aug{\mathbf{E}}_{l}^{\text{in}}$, $ l\in[0,\front{\eta}-1]\cap\mathbb{Z}$.
The optimization module extracts the first $\back{\eta}=\text{min}(\front{\eta},\maxop{\eta})$ trajectories, where $\maxop{\eta}$ is an integer chosen by the user. It optimizes the trajectory coefficients $\aug{\mathbf{E}}_l$, $\forall l\in\{0\dots\back{\eta}-1\}$, to obtain a short-term trajectory ending at the intermediate goal
\begin{align}
    \aug{\vb{p}}^{\text{inter}}=\aug{\mathbf{E}}_{\back{\eta}-1}^{\text{in}\top} \vb{g}(T),
\end{align}
where $\aug{\mathbf{E}}^{\text{in}}_l$ represents the initial solution for $\aug{\mathbf{E}}_l$ obtained from the front end. 
Figure \ref{fig: planning} shows the trajectories before and after the optimization.
The objective function is
\begin{align}
    J=\sum_{l=0}^{\back{\eta}-1}\int_{t=0}^T\norm{\aug{\mathbf{E}}_l^\top \vb{g}^{(\trajorder)}(t)}^2dt + \kappa\norm{\aug{\mathbf{E}}_{\back{\eta}-1}^\top \vb{g}(T)-\aug{\vb{p}}^\text{inter}}^2,\label{eq: objective}
\end{align}
in which the first term aims to reduce the aggressiveness by penalizing the magnitude of control input ($\trajorder$-th order derivative of the trajectory). The second term penalizes not reaching the intermediate goal.
$\kappa$ is the weight for the second term.
Setting the intermediate goal as a penalty instead of a hard constraint relaxes the optimization problem, {\cb thus improving the success rate of the optimization.}

The constraints for initial states and zero terminal higher-order states are enforced as
\begin{align}
    &\aug{\mathbf{E}}^\top_{0}\vb{g}^{(\alpha)}(0)=\aug{\vb{p}}^{\text{in}\lrangle{\alpha}}, \forall\alpha\in\{0,1,2\},\label{eq: initialcond}\\
    &\aug{\mathbf{E}}^\top_{\back{\eta}-1}\vb{g}^{(\alpha)}(T)=\vb{0}, \forall\alpha\in\{1,2\},\label{eq: terminalcond}
\end{align}
where $\aug{\vb{p}}^{\text{in}\lrangle{\alpha}}=\init{\aug{\vb{p}}},\init{\aug{\vb{v}}},\init{\aug{\vb{a}}}$ for $\alpha=0,1,2$, respectively. 
We would like the robot to stop at the intermediate goal. In case the robot cannot find a feasible trajectory for the next few iterations, it can decelerate and stop at a temporarily safe location until it reaches a feasible solution again. 
{\cb We also need to ensure the continuity between consecutive trajectories:}
\begin{align}
    \aug{\mathbf{E}}^\top_{l}g^{(\alpha)}(T)=\aug{\mathbf{E}}^\top_{l+1}\vb{g}^{(\alpha)}(0),\forall \alpha\in\{0,1,2\},l\in\{0,\dots\back{\eta}-2\}.\label{eq: continuity}
\end{align}

The dynamic feasibility constraints are enforced using the control points since the velocity and acceleration control points can be expressed as the linear combinations of the trajectory coefficients. The constraint equations are in the same form as Eqn. (\ref{eq: poscontrolpoint}-\ref{eq: acccontrolpoint}), but enforced in three dimensions.

For non-collision constraints, we apply the plane separation technique introduced in \cite{tordesillas2021mader} to the 2-D case.
{\cb The minimum set of vertices of the inflated convex hull enclosing the collaborating robot/obstacle $j$ at the planning interval $l$ is denoted as $\Theta_{j,l}$ (the vertices of the obstacles are constant with respect to $l$}).
A line parameterized by $\gbf{\pi}_{j,l}\in\mathbb{R}^2$ and $d_{j,l}\in\mathbb{R}$ is used to separate the vertices of the obstacle or robot $j$ and those of the planning robot using inequalities
\begin{align}
    &\gbf{\pi}_{j,l}^\top\gbf{\theta}+d_{j,l}>0,\; \forall\gbf{\theta}\in\Theta_{j,l},\label{eq: separate1}\\
    &\gbf{\pi}_{j,l}^\top\posctrlpt+d_{j,l}<0,\;\forall\posctrlpt\in\mathcal{Q}_{l},\label{eq: separate2}
\end{align}
$\forall j\in\myset{I}_{m+n}\backslash i$, $l\in\{0,\dots, \back{\eta}-1\}$. 

It is also necessary to add non-entanglement constraints in the optimization to prevent the robot from taking an unallowable shortcut. In Figure \ref{fig: planning}, the dotted rainbow curve is an unallowable shortcut because robot $i$ should not cross $r_{j,1}$ in this case.
{\cb A non-crossing constraint can be added in a similar way to (\ref{eq: separate1}-\ref{eq: separate2}) when using a line to separate the vertices of the non-crossing virtual segment from those of the planning robot's trajectory.}

The overall optimization problem 
is a nonconvex quadratically constrained quadratic program (QCQP) if we are optimizing over both the trajectory coefficients $\aug{\mathbf{E}}_l$ and the separating line parameters, $\gbf{\pi}_{j,l}$ and $d_{j,l}$.
To reduce the computational burden, we fix the values of the line parameters by solving (\ref{eq: separate1}-\ref{eq: separate2}) using the trajectory coefficients obtained from the front-end planner, $\init{\aug{\mathbf{E}}_l}$.
{\cb Hence, the optimization problem with only the trajectory coefficients as the decision variables is
\begin{align}
    &\!\min_{\aug{\mathbf{E}}_l}\sum_{l=0}^{\back{\eta}-1}\int_{t=0}^T\norm{\aug{\mathbf{E}}_l^\top \vb{g}^{(\trajorder)}(t)}^2dt + \kappa\norm{\aug{\mathbf{E}}_{\back{\eta}-1}^\top \vb{g}(T)-\aug{\vb{p}}^\text{inter}}^2,\nonumber\\
    &\text{s.t.}\; \nonumber\\
    &\posctrlpt\in\aug{\mathcal{W}},\; \forall\posctrlpt\in\aug{\mathcal{Q}}_l,\forall l,\nonumber\\
    &[\minop{v}^\text{x}\;\minop{v}^\text{y}\;\minop{v}^\text{z}]^\top\leq\vb{v}\leq[\maxop{v}^\text{x}\;\maxop{v}^\text{y}\;\maxop{v}^\text{z}]^\top,\; \forall\vb{v}\in\aug{\mathcal{V}}_l,\forall l,\nonumber\\
    &[\minop{a}^\text{x}\;\minop{a}^\text{y}\;\minop{a}^\text{z}]^\top\leq\vb{a}\leq[\maxop{a}^\text{x}\;\maxop{a}^\text{y}\;\maxop{a}^\text{z}]^\top,\; \forall\vb{a}\in\aug{\mathcal{A}}_l,\forall l,\nonumber\\
    &\aug{\mathbf{E}}^\top_{0}\vb{g}^{(\alpha)}(0)=\aug{\vb{p}}^{\text{in}\lrangle{\alpha}}, \forall\alpha\in\{0,1,2\},\nonumber\\
    &\aug{\mathbf{E}}^\top_{\back{\eta}-1}\vb{g}^{(\alpha)}(T)=\vb{0}, \forall\alpha\in\{1,2\},\nonumber\\
    &\aug{\mathbf{E}}^\top_{l}g^{(\alpha)}(T)=\aug{\mathbf{E}}^\top_{l+1}\vb{g}^{(\alpha)}(0),\forall \alpha\in\{0,1,2\},\forall l\in 0\cup\myset{I}_{\back{\eta}-2},\nonumber\\  
    &\gbf{\pi}_{j,l}^\top\gbf{\theta}+d_{j,l}>0,\; \forall\gbf{\theta}\in\Theta_{j,l},\forall j \in \myset{I}_{m+n}\backslash i, \forall l,\nonumber\\
    &\gbf{\pi}_{j,l}^\top\posctrlpt+d_{j,l}<0,\;\forall\posctrlpt\in\mathcal{Q}_{l},\forall j \in \myset{I}_{m+n}\backslash i,\forall l.\nonumber  
\end{align}
For brevity, in the above, $\forall l$ indicates $\forall l\in 0\cup\myset{I}_{\back{\eta}-1}$. The optimization problem is a quadratic program with a much lower complexity.}

\section{Complexity Analysis}
\subsection{Complexity of Homotopy Update} \label{subsec: complexity homotopy}
We analyze the time complexity of the homotopy update procedure (Algorithm \ref{alg: homotopyupdate}). 
We define $\maxop{\omega}$ to be the maximum possible number of obstacle-related entries, $o_{j,f}$, in $h(k)$ at any time $k$. 
$\maxop{\omega}$ depends on the cable length and the obstacle shapes (the minimum distance between any two vertices of the obstacles). It is generally independent of the number of obstacles in the workspace.
As the robots avoid entanglements based on the two-entry rule, the maximum number of robot-related entries, $r_{j,f}$, $j\in\myset{I}_n$, $f\in\{0,1\}$, in $h(k)$ should be $n$.
The time complexity of updating the word (Algorithm \ref{alg: crossing}) is dominated by the number of crossing checks.
{\cb The virtual segments of each robot are defined by at most $\maxop{\omega}$ contact points, and crossing $m$ static obstacle can be checked in $\bigO(m)$ time.
Hence, the update of the word can be conducted in $\bigO(n\maxop{\omega}+m)$ time.
To implement the reduction procedure (Algorithm \ref{alg: reduction}), three nested loops are required to inspect each entry of $h(k)$. Hence, the time complexity of the reduction procedure is $\bigO((n+\maxop{\omega})^3)$.}
Similarly, the contact point update procedure (Algorithm \ref{alg: funnel}) can be run in $\bigO(n+\maxop{\omega})$ time.
Therefore, the time complexity of the entire homotopy update procedure is $\bigO(m+(n+\maxop{\omega})^3)$.

The memory complexity of the homotopy update depends on the storage of all the virtual segments. Hence, it is $\bigO(n\maxop{\omega}+m)$.

\subsection{Complexity of a Planning Iteration} \label{subsec: search complexity}
Both the time complexities of the front-end kinodynamic planning (Section \ref{sec: search}) and the backend optimization (Section \ref{sec: opt}) are analyzed.
The time complexity of the kinodynamic search is the product of the total number of candidate trajectories (successive nodes) generated and the complexity of evaluating each trajectory.
In the worst case, a total of $\bigO(\epsilon(n+\maxop{\omega})\sigma_\text{u})$ candidate trajectories are generated in the homotopy-augmented graph, where $\epsilon$ is the total number of grids in the workspace $\myset{W}$.
$\epsilon$ is dependent on the total area of the workspace and the grid size used for discretization.
For each candidate trajectory, checking the workspace and dynamic feasibility (\ref{eq: poscontrolpoint})-(\ref{eq: acccontrolpoint}) can be done in $\bigO(1)$ time.
The collision avoidance requirement (\ref{eq: collision avoidance}) can be checked in $\bigO(m+n)$ time.
The time complexity of checking cable-related requirements (Algorithm \ref{alg: entangle}) is dominated by the homotopy update procedure (line 5), which has to be executed $\sigma$ times;
therefore, the time complexity of Algorithm \ref{alg: entangle} is $\bigO(\sigma (m+(n+\maxop{\omega})^3))$.
Hence, the worst-case time complexity of the kinodynamic search is
$\bigO(\epsilon\sigma_\text{u}\sigma (n+\maxop{\omega})(m+(n+\maxop{\omega})^3))$.

The time complexity of solving the optimization problem using the interior point method is $\bigO((\maxop{\eta}\trajorder)^3)$ \cite{ye1989extension}, as there are $3\maxop{\eta}(\trajorder+1)$ optimization variables.
Therefore, the time complexity of one planning iteration is $\bigO((\maxop{\eta}\trajorder)^3+\epsilon\sigma_\text{u}\sigma (n+\maxop{\omega})(m+(n+\maxop{\omega})^3))$.
Given that the independent parameters $\epsilon$, $\maxop{\omega}$, $\sigma$, $\sigma_\text{u}$, $\maxop{\eta}$ and $\trajorder$ are fixed, the time complexity simplifies to $\bigO(nm+n^4)$, which is linear in the number of obstacles and quartic in the number of robots.

The memory complexity of planning is dominated by the memory to store the valid graph nodes for the kinodynamic search. Each node has a memory complexity of $\bigO(n+\maxop{\omega}+\trajorder)$ due to the need to store the trajectory coefficients and the homotopy representation.
Adding the memory to store a copy of the virtual segments, the memory complexity of a planning iteration is $\bigO(\epsilon(n+\maxop{\omega})(n+\maxop{\omega}+\trajorder)+m)$, which can be simplified to $\bigO(n^2+m)$.

\subsection{Communication Complexity}
In the message transmitted by each robot per iteration, {\cb the lengths of the robot's position,  the list of contact points, and the robot's future trajectory are $\bigO(1)$, $\bigO(\maxop{\omega})$, and $\bigO(\maxop{\eta}\trajorder)$, respectively.}
Considering a planning frequency of $f_\text{com}$, the amount of data transmitted by each robot is $\bigO(f_\text{com}(\maxop{\omega}+\maxop{\eta}\trajorder))$ per unit time, 
while the data received by each robot per unit time is $\bigO(nf_\text{com}(\maxop{\omega}+\maxop{\eta}\trajorder))$.

\section{Simulation}\label{sec: simulation}
The proposed multi-robot homotopy representation and the planning framework are implemented in C++ programming language.
During the simulation and experiments, each robot runs an independent program under the framework of Robot Operating System (ROS). The communication between robots is realized by the publisher-subscriber utility of the ROS network.
The trajectory optimization described in Section \ref{sec: opt} is solved using the commercial solver Gurobi\footnote{https://www.gurobi.com/}.
The processor used for simulations in Sections \ref{subsec: single} and \ref{subsec: multi without} is Intel i7-8750H while that used in Section \ref{subsec: multi with} is Intel i7-8550U. 
The codes for the compared methods are also implemented by ourselves in C++ and optimized to the best of our ability.
In all simulations, the parameters are chosen as $T=0.5$ s, $\maxop{\vb{v}}=[2.0, 2.0, 2.0]^\top$ m/s, $\maxop{\vb{a}}=[3.0, 3.0, 3.0]^\top$ $\text{m/}\text{s}^2$, $\maxop{u}=5.0$ $\text{m/}\text{s}^3$, $\minop{\vb{v}}=-\maxop{\vb{v}}$, $\minop{\vb{a}}=-\maxop{\vb{a}}$. The grid size for the kinodynamic graph search is $0.2\times0.2$ m. 
Video of the simulation can be viewed in the supplementary material or online\footnote{https://youtu.be/8b1RlDvQsi0}.

\subsection{\cb Single Robot Trajectory Planning}\label{subsec: single}
\begin{figure}[!t]
\centering
\includegraphics[width=\linewidth]{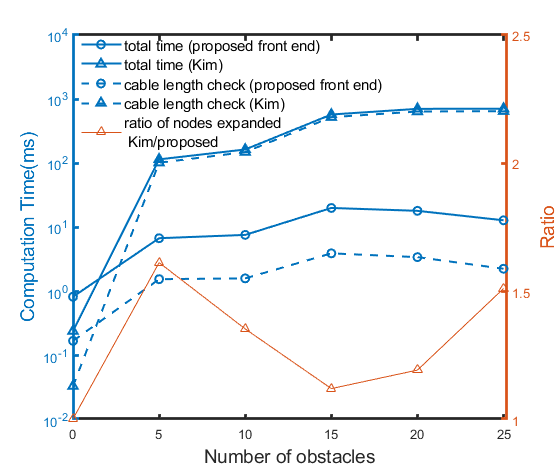}
\caption{\footnotesize Comparison of computation time for single-robot planning. The proposed front end is significantly faster than Kim's approach.}
\label{fig: timecompare}
\end{figure}

\begin{figure}[!t]
\centering
\includegraphics[width=\linewidth]{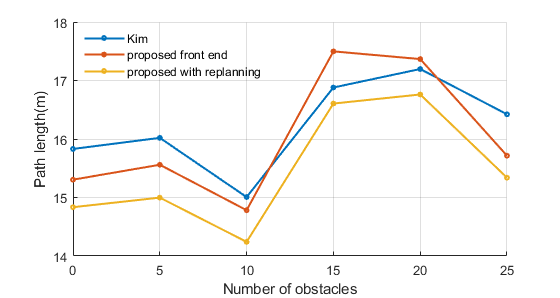}
\caption{\footnotesize Comparison of path/trajectory lengths for single-robot planning. The proposed planning framework with replanning achieves a shorter trajectory length compared to the proposed front end only and Kim's approach.}
\label{fig: lengthcompare}
\end{figure}
\begin{figure}[!t]
\centering
\includegraphics[width=0.8\linewidth]{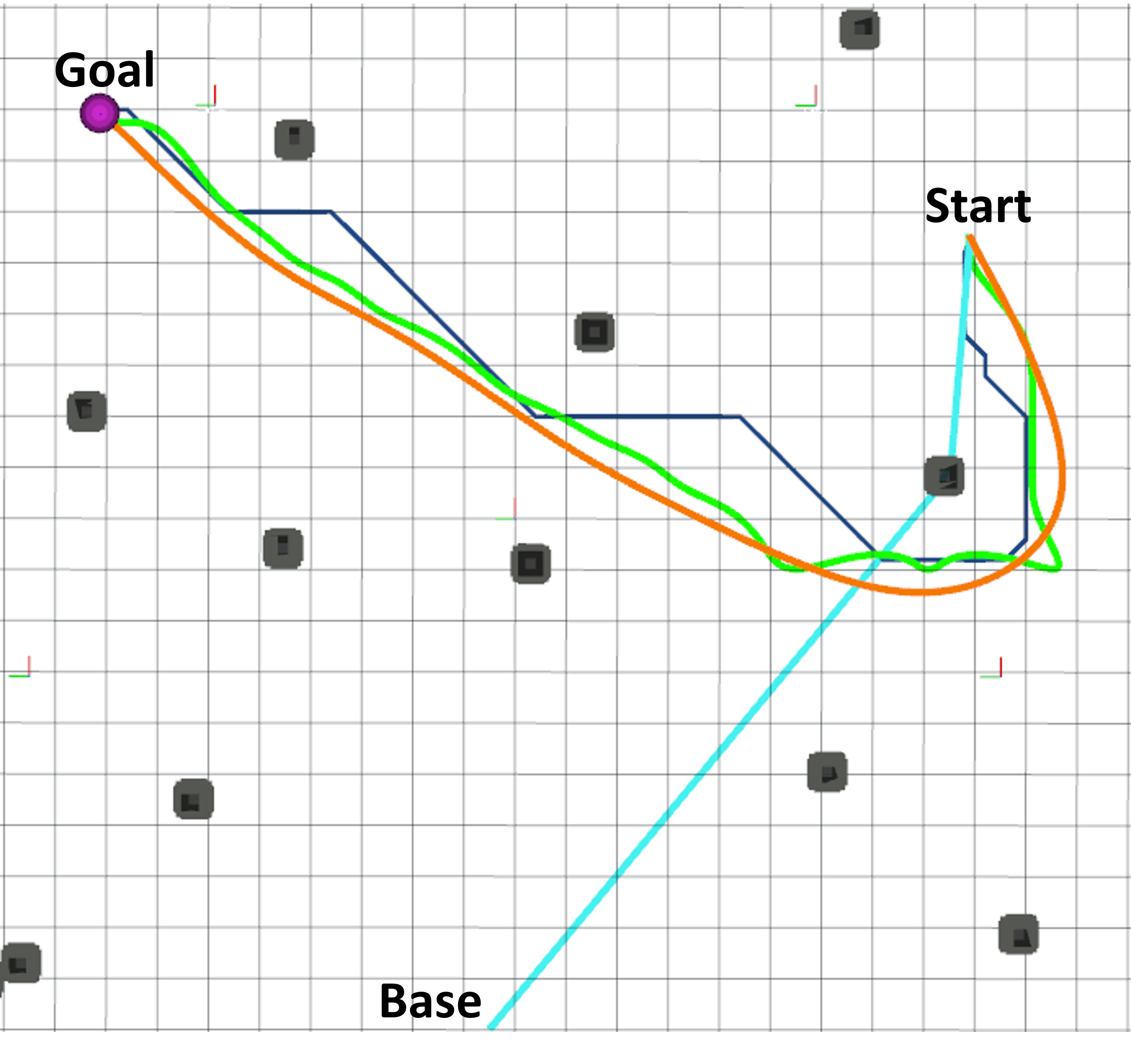}
\caption{\footnotesize The cyan line shows the shortest cable configuration at the start. The blue path is generated by Kim's method, the green curve is the trajectory generated from the proposed front-end trajectory searching. The orange curve represents the actual trajectory executed using the proposed planning framework.}
\label{fig: pathcompare}
\end{figure}

We first apply the presented approach to the planning problem for a single tethered robot in an obstacle-rich environment.
{\cb Specifically, the front-end kinodynamic trajectory search technique (Section \ref{sec: search}) is used for point-to-point trajectory generation and the entire framework (Section \ref{sec: formulation}) is used for online trajectory replanning.}
We compare our approach with Kim\cite{kim2014path} which generates piece-wise linear paths using homotopy-augmented grid-based A*.
%the code is implemented by us in C++ and optimized to the best of our ability.
The simulation environment is a $30\times30$ m 2-D area with different numbers of randomly placed obstacles.
{\cb The grid size for Kim's graph planner is set to two times the grid size of the proposed kinodynamic planner to ensure comparable amount of expanded graph nodes for both planners. When using the same grid size, the kinodynamic planner usually expands much fewer nodes than a purely grid-based A*, because the successive trajectories do not necessarily fall in neighbouring nodes.}
We randomly generate $100$ target points in the environment. Using both the proposed front-end trajectory searching and Kim's method, we generate paths that transit between target points (running each method once for every target point).
Figure \ref{fig: timecompare} shows the average computation time for both approaches and the ratio of the average number of nodes expanded in Kim's method to that in the proposed method. 
Although the numbers of the expanded nodes are comparable for both approaches, {\cb the computation time for Kim's method is at least an order of magnitude longer than our method except when no obstacle exists }(which is equivalent to planning without a tether).
In Kim's approach, a large proportion of time is spent on checking the cable length requirement for a robot position, which uses the expensive curve shortening technique discussed in Section \ref{sec: prelim}.
In comparison, our contact point determination procedure consumes much less time and hence contributes to the efficiency of the proposed method.
Figure \ref{fig: lengthcompare} shows the average path/trajectory lengths for both approaches and the average lengths of trajectories actually executed by the robot using the proposed planning framework with online trajectory replanning.
We observe that both kinodynamic trajectory searching and grid-based graph planning generate paths with comparable lengths.
{\cb The paths generated using Kim's method are optimal in the grid map, 
while the trajectories generated by our kinodynamic planner are in the continuous space (not required to pass through the grid center) and hence can have shorter lengths.}
The fast computation of the trajectory searching enables frequent online replanning which further refines the trajectory to be shorter and smoother.
Figure \ref{fig: pathcompare} shows one planning instance, where the actual trajectory (orange curve, $27.7$ m) is shorter and smoother than the paths generated from Kim's method (blue line, $28.3$ m) and kinodynamic search (green curve, $28.7$ m).

\subsection{\cb Multi-Robot Trajectory Planning without Static Obstacles}\label{subsec: multi without}
{\cb Static obstacles were not considered in all of the existing works on tethered multi-robot planning.}
Hence, we consider an obstacle-free 3-D environment where multiple UAVs are initially equally spaced on a circle of radius $10$ m, {\cb with $100$ sets of random goals sent to the robots.
A mission is successful when all robots reach their goals.}
We compare our planning framework with Hert's centralized method \cite{hert1999motion} that generates piece-wise linear paths for each robot.
In Hert's original approach, a point robot is considered, but we modify the approach to handle non-zero radii of the robots for collision avoidance. 
We also change the shortest path finding algorithm from a geometric approach \cite{PAPADIMITRIOU1985259} to grid-based A* for efficiency.

\begin{figure}[!t]
\centering
\includegraphics[width=\linewidth]{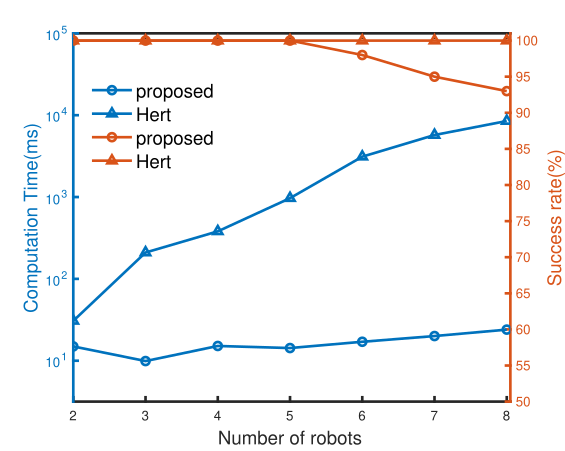}
\caption{\footnotesize   Comparison of computation time and success rate for multi-robot planning.}
\label{fig: mtlp_compare}
\end{figure}

\begin{figure}[!t]
\centering
\includegraphics[width=0.75\linewidth]{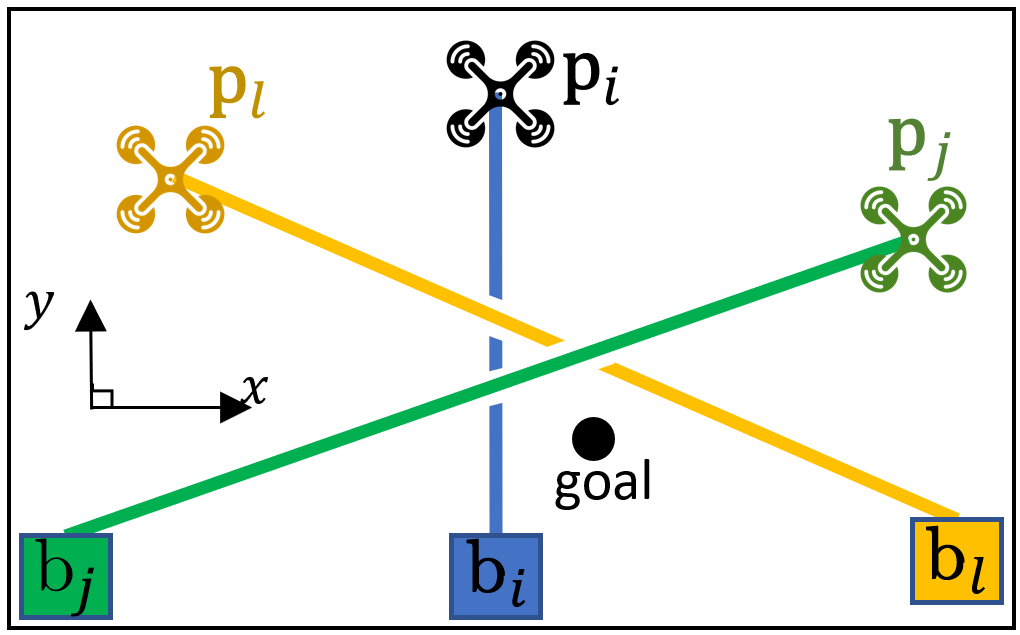}
\caption{\footnotesize   Illustration of a deadlock situation. Robot $i$'s cable has been crossed by both robot $j$ and robot $l$. The word of robot $i$ is $h_i=r_{l,0}r_{j,0}$. Unless robot $j$ and robot $l$ move, any route planned by robot $i$ to reach its goal will fail the entanglement check based on the two-entry rule.}
\label{fig: deadlock}
\end{figure}

Figure \ref{fig: mtlp_compare} plots the average computation time and the success rates of both approaches.
The computation time for our approach refers to the time taken by a robot to generate a feasible trajectory in one planning iteration.
As seen, for more than $2$ robots, the computation time for Hert's approach is at least an order of magnitude longer compared to our approach.
While only a small increase in computation time is observed in our approach, the computation time for Hert's approach increases significantly from $30.7$ ms for $2$ robots to $8.49$ s for $8$ robots.
We note that the implementation of Hert's approach by us has already achieved significant speedup compared to the results in \cite{hert1999motion} (more than $1000$ s for $5$ robots), likely due to the use of modern computational geometry libraries, a more efficient pathfinding algorithm, and a faster processor.
{\cb Our approach is of $\bigO(n^4)$ time complexity, which is consistent with Hert's approach.
However, due to a tight and straight cable model, in Hert's method, intersections among cables must be checked for all potential paths, which is computationally expensive.}

The success rate of Hert's approach is consistently $100\%$, while our approach fails occasionally for more than $5$ robots.
However, this does not indicate less effectiveness of our approach, because the cable models and the types of entanglements considered are different in both approaches; the solution from one approach may not be a feasible solution for the other approach.
The failures of our method are due to the occurrences of deadlocks. 
{\cb A common deadlock is illustrated in Figure \ref{fig: deadlock}, where all possible routes for robot $i$ to reach its goal fail the entanglement check.
This is a drawback of decentralized planning in which the robots only plan their own trajectories and do not consider whether feasible solutions exist for the others.}

\begin{figure}[!t]
\centering
\includegraphics[width=\linewidth]{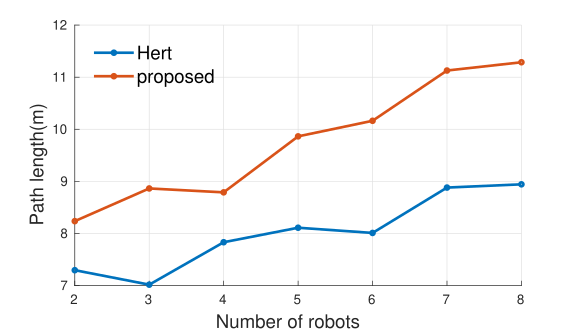}
\caption{\footnotesize  Comparison of path/trajectory lengths for multi-robot planning.}
\label{fig: mtlp_length_compare}
\end{figure}

{\cb The average path length of each robot is shorter for Hert's approach, as shown in Figure \ref{fig: mtlp_length_compare}.}
The cable model considered in Hert's work allows a robot to move vertically below the cables of other robots while avoiding contacts, 
while our approach restricts such paths and requires the robots to make a detour on the horizontal plane when necessary.
This difference is illustrated in Figure \ref{fig: mtlp},
where the blue path is generated using Hert's method while the orange curve is the trajectory generated using our approach.
{\cb In practice, it is difficult to ensure the full tightness and straightness of a cable.} Hence, moving below a cable presents a greater risk of collision than moving horizontally.

\begin{figure}[!t]
\centering
\includegraphics[width=0.9\linewidth]{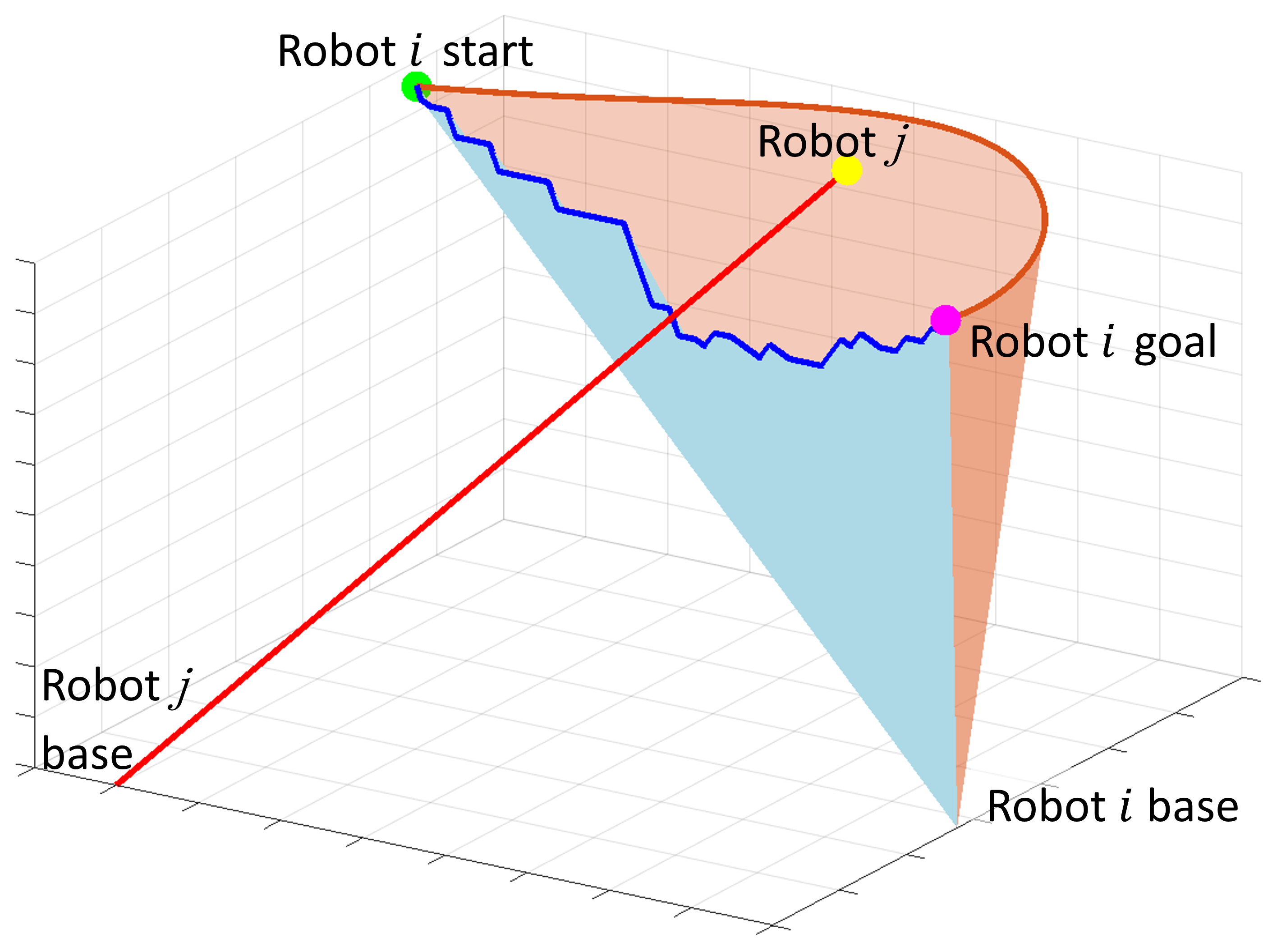}
\caption{\footnotesize Robot $i$ starts at the green point and reaches the magenta point. 
The blue path is generated using Hert's method while the orange curve is the trajectory of the robot using our approach.
The red line simulates the cable of robot $j$ if the cable is fully tight.
The light blue and orange sections are the areas swept by the cable of robot $i$ during its movement from start to goal, considering a tight cable model.
In this case, both approaches are able to generate paths that avoid intersection between straight cables, but Hert's approach requires robot $i$ to move below the cable of robot $j$.  }
\label{fig: mtlp}
\end{figure}

\subsection{\cb Multi-Robot Trajectory Planning with Static Obstacles}\label{subsec: multi with}
\begin{figure*}
	\centering
    \subcaptionbox{ \label{fig: cir_1}}[0.32\linewidth]{\includegraphics[width=\linewidth]{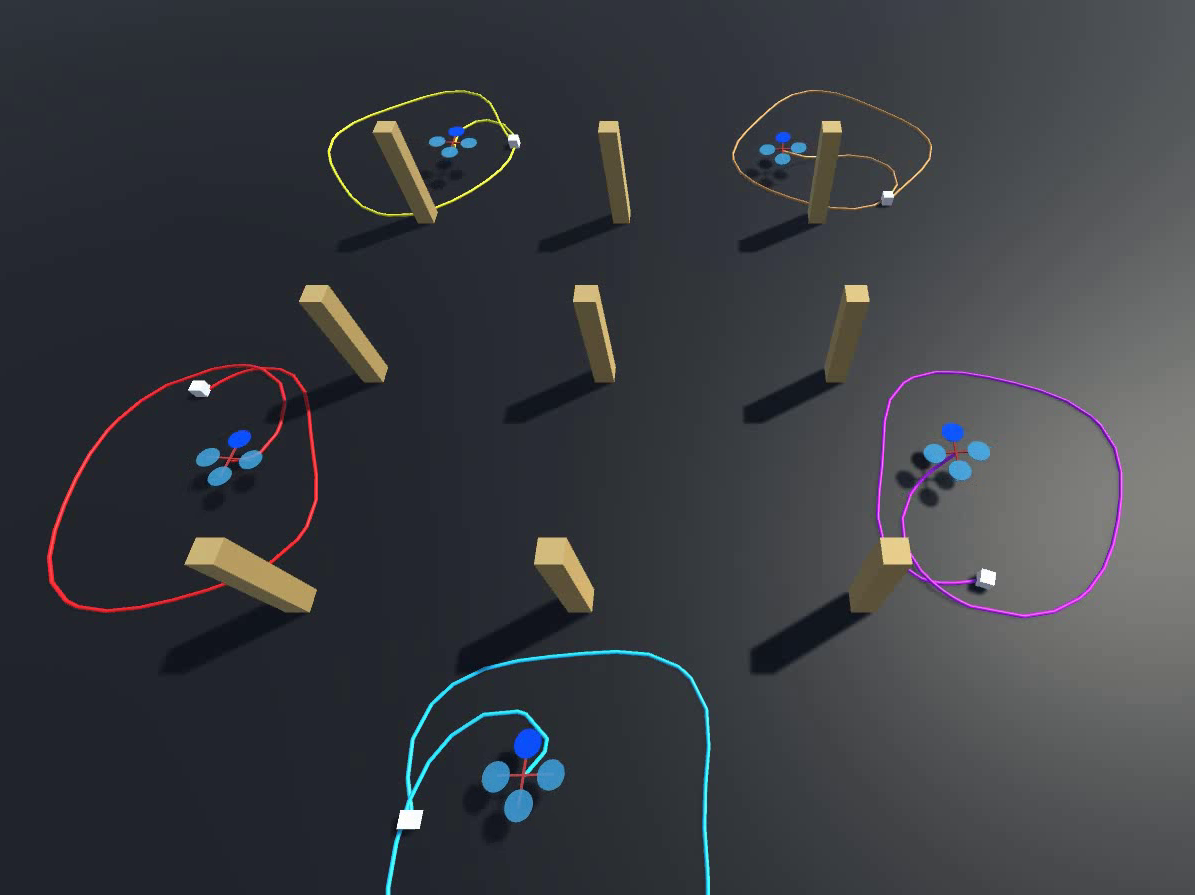}}
    \subcaptionbox{ \label{fig: cir_2}}[0.32\linewidth]{\includegraphics[width=\linewidth]{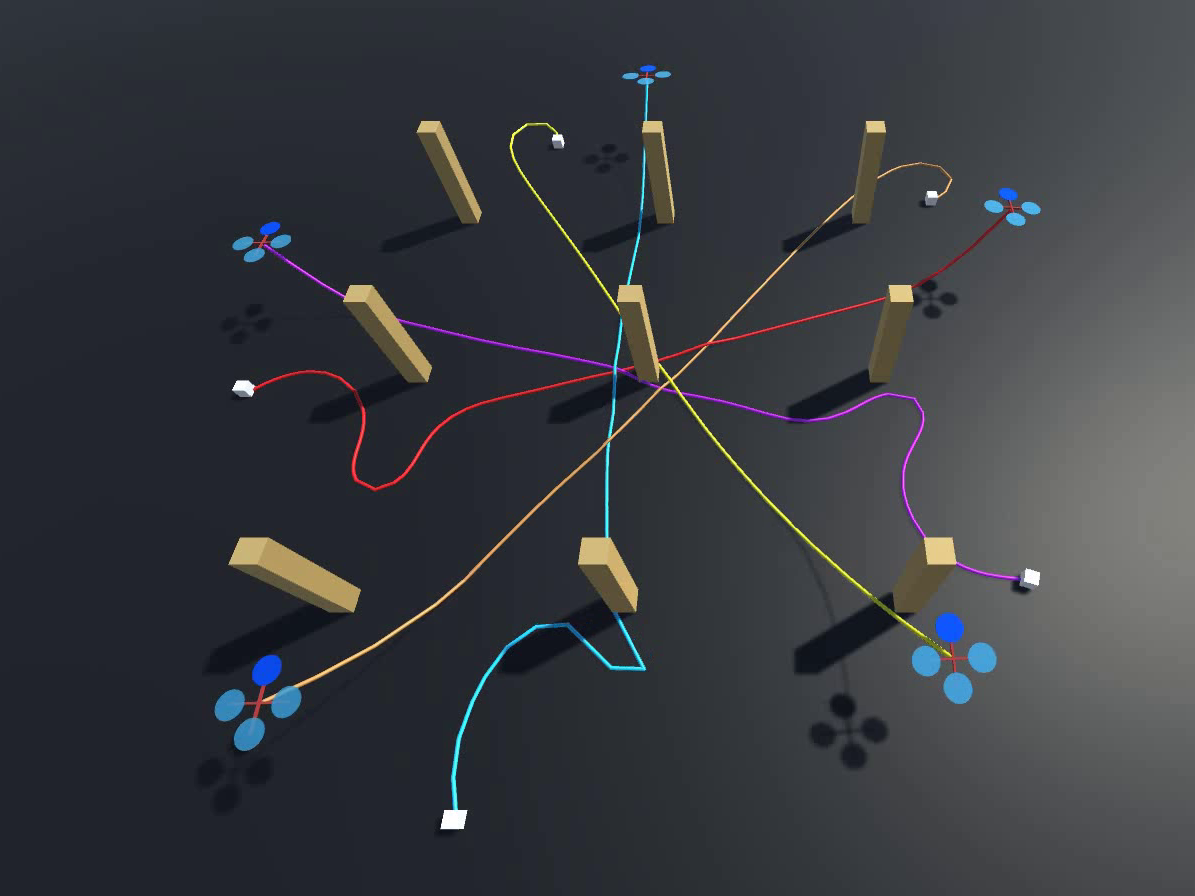}}
    \subcaptionbox{ \label{fig: cir_3}}[0.32\linewidth]{\includegraphics[width=\linewidth]{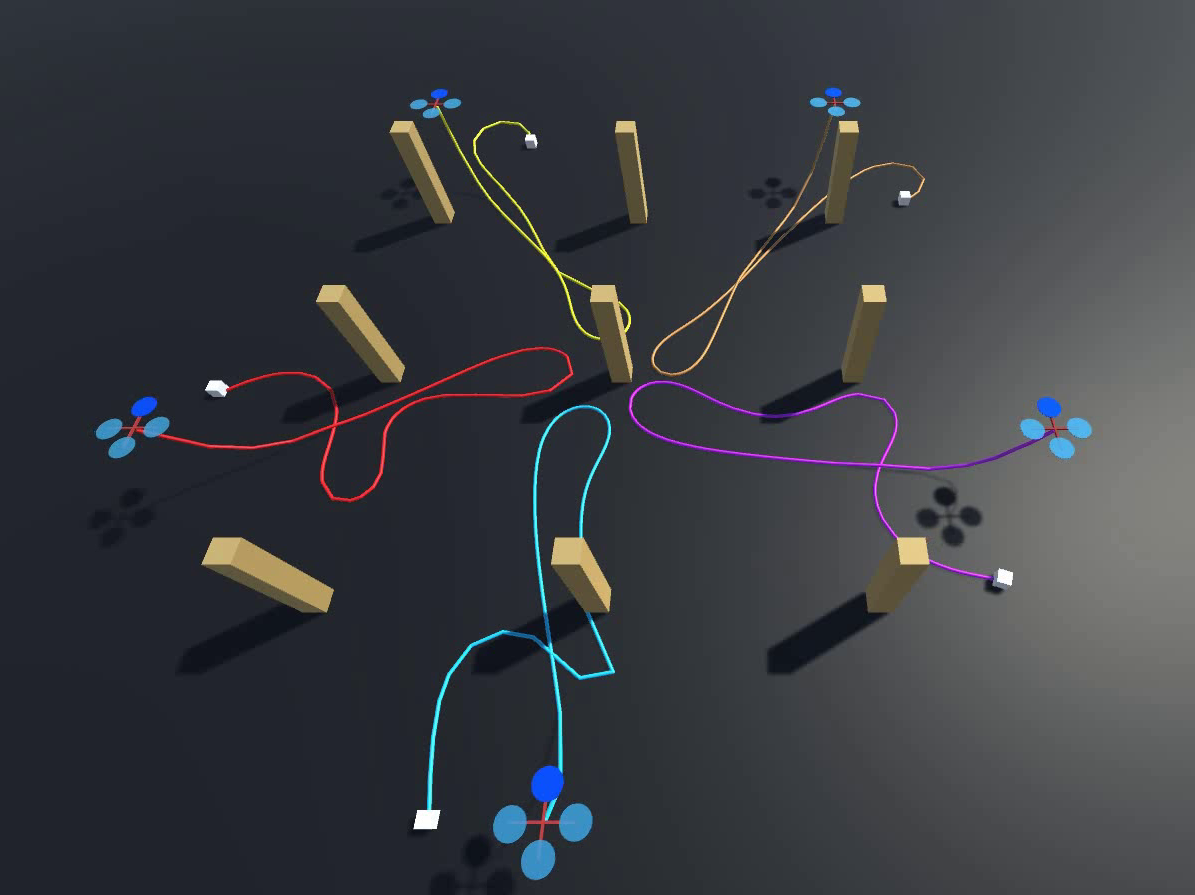}}
	\caption{ \footnotesize 5 UAVs planning in a workspace with $9$ static obstacles. (a) UAVs take off from their starting positions. (b) UAVs reach their targets opposite to their starting positions on the circle. (c) UAVs return to their starting positions without entanglements.}  \label{fig: multicircle}
\end{figure*}
We conduct simulation of multiple tethered UAVs working in an environment placed with $9$ square obstacles, as shown in Figure \ref{fig: multicircle}.
The starting positions of the UAVs are evenly distributed on a circle of radius $10$ m. Each UAV is tasked to move to the opposite point on the circle and then move back to its starting point, hence crossing of other UAVs' cables and passing through obstacles are inevitable.
{\cb The mission is considered successful if all robots return to their starting points at the end of the mission.}
UAV and cable dynamics are simulated using Unity game engine with AGX Dynamics physics plugin\footnote{https://www.algoryx.se/agx-dynamics/}; collisions among cables, UAVs and static obstacles are simulated to detect contacts and entanglements. 
Forces and torques commands for UAV are computed in ROS using the trajectory output from the proposed planner and the UAV states obtained from Unity.
We conduct $30$ simulation runs for numbers of robots ranging from $2$ to $8$ and record the computation time of each execution of front-end planning and back-end optimization, {\cb as well as the total time taken for each simulation.}
The results are shown in Table \ref{tab: multiple}. {\cb It is observed that}:
\begin{table}
\centering
\begin{tabular}{|c|c|c|c|c|}
\hline
\begin{tabular}[c]{@{}c@{}}Num. \\ of \\ Robots\end{tabular} & \begin{tabular}[c]{@{}c@{}}Avg. Comp. \\ Time \\ Front End \\ (ms)\end{tabular} & \begin{tabular}[c]{@{}c@{}}Avg Comp. \\ Time \\ Back End \\ (ms)\end{tabular} & \begin{tabular}[c]{@{}c@{}}Avg Time\\ per \\ Mission \\ (s)\end{tabular} & \begin{tabular}[c]{@{}c@{}}Mission\\ Success \\ Rate \\ (\%)\end{tabular} \\ \hline
$2$                                                          & $11.11$                                                                         & $19.40$                                                                       & $32.91$                                                                  & $100$                                                                     \\ \hline
$3$                                                          & $10.74$                                                                         & $24.15$                                                                       & $29.86$                                                                  & $100$                                                                     \\ \hline
$4$                                                          & $18.61$                                                                         & $24.84$                                                                       & $38.26$                                                                  & $100$                                                                     \\ \hline
$5$                                                          & $20.30$                                                                         & $22.91$                                                                       & $43.37$                                                                  & $100$                                                                     \\ \hline
$6$                                                          & $30.31$                                                                         & $21.22$                                                                       & $55.85$                                                                  & $83.3$                                                                    \\ \hline
$7$                                                          & $30.50$                                                                         & $21.82$                                                                       & $61.62$                                                                  & $93.3$                                                                    \\ \hline
$8$                                                          & $38.79$                                                                         & $22.87$                                                                       & $69.81$                                                                  & $80.0$                                                                    \\ \hline
\end{tabular}
\caption{Results for Multiple Tethered UAVs Planning}
\label{tab: multiple}
\end{table}
\begin{itemize}
  \item {\cb For front-end trajectory finding, the computation time increases with the number of robots.
  This is mainly due to the increasing possibility of blocking the routes by the collaborating robots or their cables} (the route can be virtually blocked by a cable if crossing this cable is risking entanglement), causing the planner to take a longer time to find a detour and a feasible trajectory.
  However, the increase in computation time is small ($<30$ ms) and the planner still satisfies the real-time requirement for $8$ robots.
  \item The back-end optimization has relatively consistent computation time ($\sim 22$ ms), because the number of polynomial trajectories to be optimized is fixed regardless of the number of robots.
  \item The average computation time for one planning iteration (including both the front end and the back end) is well below $100$ ms and suitable for real-time replanning during flights.
  \item {\color{black}The time to complete the mission increases with the number of robots because more time is spent on waiting for other robots to move so that the cables no longer block the only feasible route.}
  \item The planner achieves $100\%$ mission success rate for numbers of robots less than or equal to $5$.
  As the number of robots increases, 
  the success rate drops but still maintains above $80\%$ for $8$ robots.
  Similar to the cases in Section \ref{subsec: multi without}, the failures are due to deadlocks, which are more likely to occur in a cluttered environment.
\end{itemize}

In all of the simulation runs, no entanglements are observed, showing the effectiveness of the proposed homotopy representation and the two-entry rule in detection and prevention of entanglements. However, the proposed two-entry rule is conservative in evaluating the risk of entanglements. For example, in reality, {\cb robot $i$ in Figure \ref{fig: deadlock} may be able to reach its goal with a long cable}; however, such a route is prohibited by the two-entry rule because it has to cross the cables of robot $j$ and $l$.
In essence, the proposed method sacrifices the success of a mission to guarantee safety, which is reasonable in many safety-critical applications.
{\color{black}Additional features may be implemented to improve the success rate and resolve deadlocks. For example, in Figure \ref{fig: deadlock}, robot $i$ may request other robots to uncross its cable before moving to the goal. Alternatively, feasible trajectories of the robots can be computed in a centralized manner, before they are sent to each robot for optimisation. Nevertheless, such an approach will inevitably be more computationally expensive.}

Overall, the proposed planning framework is shown to be capable of real-time execution and effective in preventing entanglement for different numbers of robots.

\section{Flight Experiments}
We conduct flight experiments using $3$ self-built small quadrotors in a $6\times6$ m indoor area. {\cb Each quadrotor is attached to a cable with a length of $7$ m connected to a ground power supply.}
Each quadrotor is equipped with an onboard computer with an Intel i7-8550U CPU, running the same ROS program as described in Section \ref{sec: simulation}.
All onboard computers are connected to the same local network through Wi-Fi.
A robust tracking controller \cite{Cai2011} is used to generate attitude and thrust commands from the target trajectory, which are sent to the low-level flight controller using DJI Onboard SDK.
The parameters for the planner are chosen as $T=0.3$ s, $\maxop{\vb{v}}=[0.7, 0.7, 0.7]^\top$ m/s, $\maxop{\vb{a}}=[2.0, 2.0, 2.0]^\top$$\text{m/}\text{s}^2$, $\maxop{u}=3.5$$\text{m/}\text{s}^3$,  $\maxop{\eta} = 8$, $\sigma^\text{u}=2$, $\sigma = 3$,  and the grid size for the kinodynamic graph search is $0.05\times0.05$ m. 
The rate of planning is $10$Hz.
The quadrotors are commanded to shuttle between two positions in the workspace, as illustrated in Figure \ref{fig: exp}, which resembles an item transportation task in a warehouse scenario.
The supplementary video shows an experiment in which each robot completes $15$ back-and-forth missions without incurring any entanglement.
{\cb The tethered power supply enables longer mission duration than the $2$-minutes flight time of the quadrotor under battery power.}
Figure \ref{fig: exp_plot} shows the command trajectories and the actual positions of a UAV during part of the experiment. The generated command trajectories show high smoothness in both X and Y axes, {\cb thus enabling good tracking performance of the robots}.
\begin{figure}[!t]
\centering
\includegraphics[width=\linewidth]{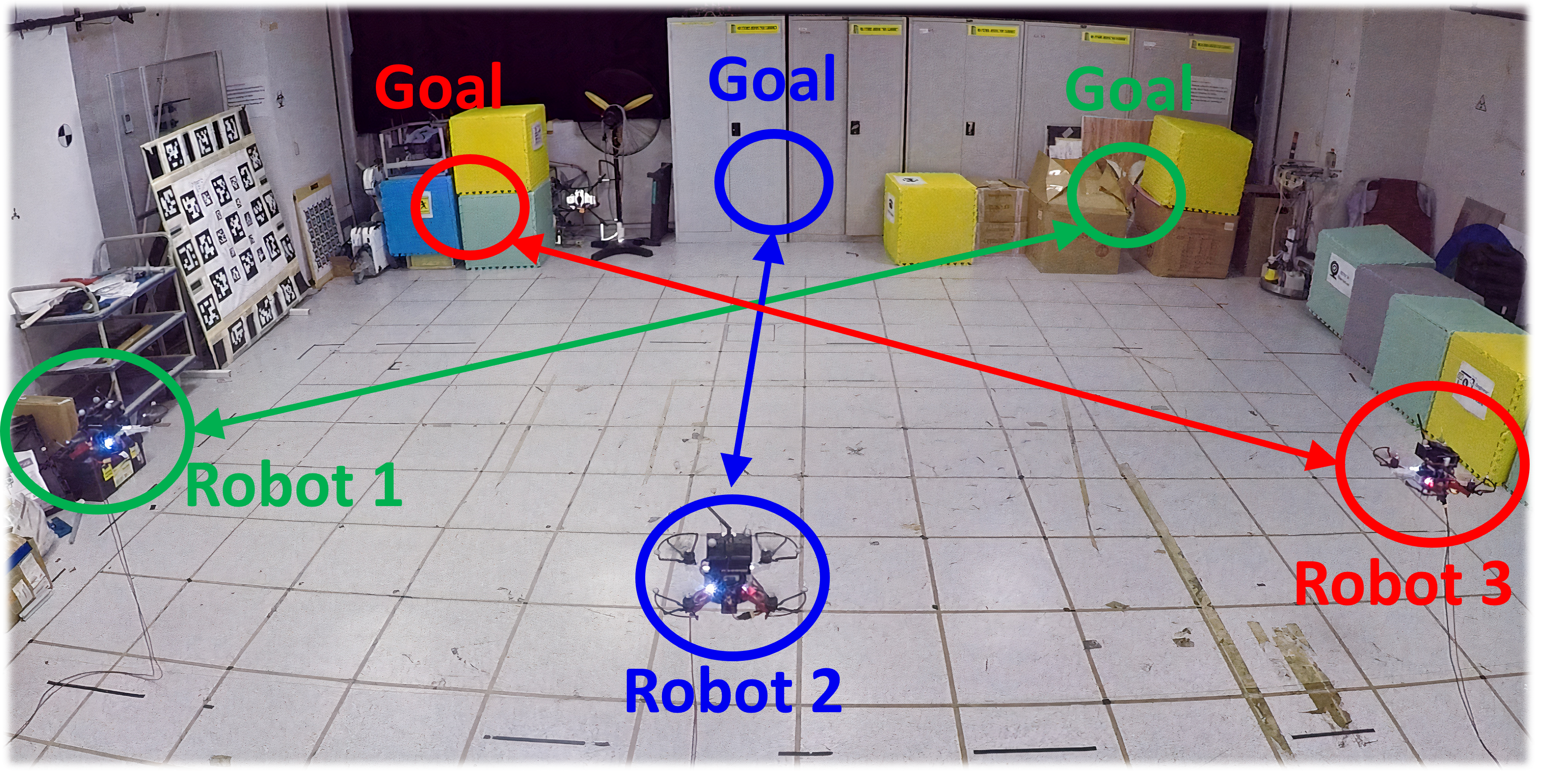}
\caption{\footnotesize Experiment using $3$ tethered UAVs.}
\label{fig: exp}
\end{figure}

\begin{figure}[!t]
\centering
\includegraphics[width=\linewidth,trim={0.6cm 0.2cm 0.2cm 0cm}]{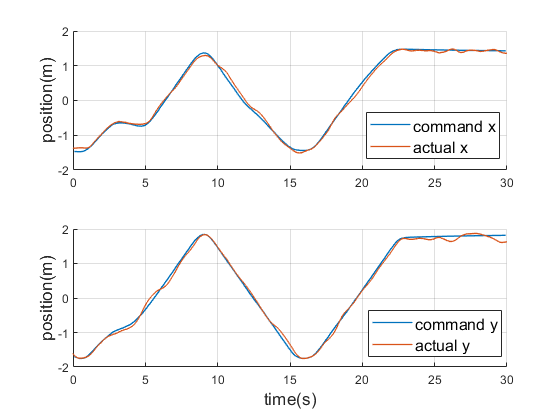}
\caption{\footnotesize Plot of command trajectories and actual positions of a UAV during the flight experiment.}
\label{fig: exp_plot}
\end{figure}

\section{Conclusion}
In this work, we presented NEPTUNE, a complete solution for the trajectory planning for multiple tethered robots in an obstacle-ridden environment.
Central to the approach is a multi-robot tether-aware representation of homotopy, which encodes the interaction among the robots and the obstacles in the form of a word, and facilitates the computation of contact points to approximate the shortened cable configuration. 
{\cb The front-end trajectory finder leverages the proposed homotopy representation to discard trajectories risking entanglements or exceeding the cable length limits.} The back-end trajectory optimizer refines the initial feasible trajectory from the front end.
{\cb Simulations in single-robot obstacle-rich and multi-robot obstacle-free environments showed improvements in computation time compared to the existing approaches.}
Simulation of challenging tasks in multi-robot obstacle-rich scenarios showed the effectiveness in entanglement prevention and the real-time capability.
Flight experiments highlighted the potential of NPETUNE in practical applications using real tethered systems.
Future works will focus on improving the success rate by introducing deadlock-resolving features and applications in a realistic warehouse scenario.

\begin{appendices}

    \section{Explanation of Homotopy Induced by Obstacles in a 3-dimensional Space} \label{apd: homotopy3d}
\begin{figure}[!t]
\centering
\includegraphics[width=\linewidth]{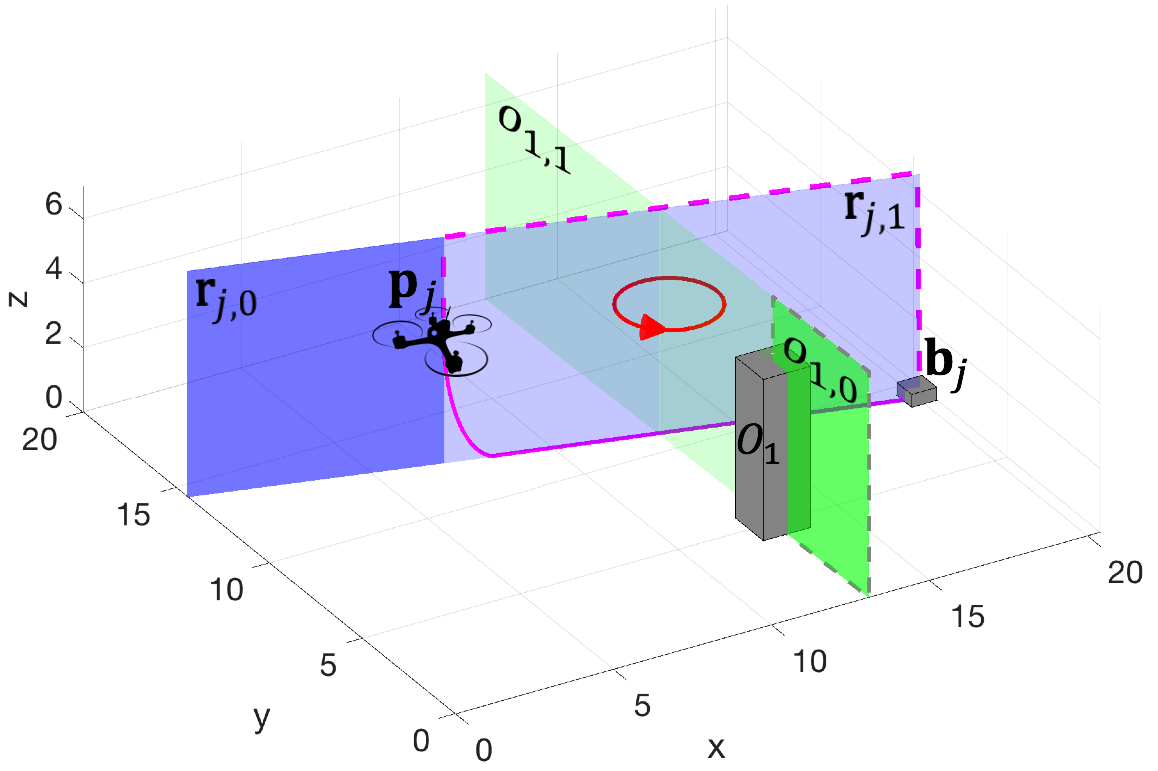}
\caption{\footnotesize A 3-D workspace consisting of a robot $j$ and static obstacle $1$. The solid magenta curve is the cable attached to the drone.}
\label{fig: 3dhomotopy}
\end{figure}
In a 2-D space, different homotopy classes are created due to the presence of obstacles (or punctures) in the space. For example,
a path turning left at an obstacle is topologically different from a path turning right at the obstacle to reach the same goal.
However, in a 3-D space, not all obstacles can induce different homotopy classes; only those containing one or more holes (those with genus equal to or more than 1) are able to do so.
In our case, both the static obstacles and the cables attached to the drones are 3-D obstacles, {\cb but they generally contain no holes.}
Therefore, we manually close those 3-D obstacles using the following procedure.
Since we restrict the planning robot to move above any other robots or obstacles,  the space above them can be considered as virtual obstacles.
We further extend the virtual obstacles along the workspace boundaries until they reach the respective bases or static obstacles.
In Figure \ref{fig: 3dhomotopy}, the magenta and grey dashed lines outline the virtual obstacles created for the robot $j$ and obstacle $1$ respectively.
Each 3-D obstacle, joined with its virtual obstacle, contains a hole,
so that different homotopy classes can be induced by paths passing through the hole and paths passing outside the hole.
Virtual 2-D manifolds can be created, {\cb as shown in Figure \ref{fig: 3dhomotopy} by the blue and green planes}. The sequence of the manifolds being intersected by a path can be recorded to identify the homotopy class of the path.
We can observe the similarity between such a construction in 3-D space and the 2-D method discussed in Section \ref{subsec: lineseg}:
given a 3-D space where all obstacles remain static,
two methods produce the same word for a path that avoids passing above the obstacles and maintains a safe distance to the obstacles.
We also gain an understanding of the loops in the 2-D case by looking at the 3-D case: the loops at the intersection between two 2-D manifolds do not physically wrap around any obstacles, hence they can be contracted to a point topologically.
As shown in Figure \ref{fig: 3dhomotopy}, the red circle that cuts through manifolds $o_{1,1}$, $r_{j,1}$ alternatively is null-homotopic.

\section{Sketch of Proof of Proposition \ref{prop: shortest}}\label{apd: proofshortest}
The proof is based on the fact that, for a path lying in the universal covering space of a workspace consisting of polygonal obstacles, the shortest homotopic path can be constructed from the vertices of the obstacles, and the start and end points \cite{HERSHBERGER199463,lee1984euclidean}.
{\cb Since only thin barrier obstacles are considered in our approach,} the vertices are the surface points $\gbf{\zeta}_{j,f}$, $j\in\myset{I}_m$, $f\in\{0,1\}$.
Under Assumption \ref{ass: initialpos}, a surface point $\gbf{\zeta}_{j,f}$ can become a point on the shortest path only when the corresponding virtual segment $o_{j,f}$ has been crossed. Hence, it is sufficient to check only the surface points of the obstacles in $h(k)$.

{\color{black}
\section{Computing Initial Z-Axis Trajectories}\label{subsec: initialz}
}
The generation of trajectories in Z-axis uses the properties of a clamped uniform b-spline.
A clamped uniform b-spline is defined by its degree $\trajorder$, a set of $\lambda+1$ control points $\{\ctrlptbspline_0, \ctrlptbspline_1, \dots,\ctrlptbspline_{\lambda}\}$ and $\lambda+\trajorder+2$ knots $\{t_0, t_1, \dots t_{\lambda+\trajorder+1}\}$, where $t_0=t_1=\dots=t_p$, $t_{\lambda+1}=t_{\lambda+2}=\dots=t_{\lambda+\trajorder+1}$, and the internal knots $t_\trajorder$ to $t_{\lambda+1}$ are equally spaced. 
%i.e., $t_{l+1}-t_l=t_l-t_{l-1}, \forall l\in[p+1,\lambda]\cap\mathbb{Z}$.
It uniquely determines $\lambda-\trajorder+1$ pieces of polynomial trajectories, each of a fixed time interval.
% where the $l$-th trajectory is of interval $[t_{\trajorder+l-1},t_{\trajorder+l}]$.
It has the following properties of our interest\cite{zhou2019robust}: (1) the trajectory defined by a uniform clamped b-spline is guaranteed to start at $\ctrlptbspline_0$ and end at $\ctrlptbspline_{\lambda}$; (2) the first $(\trajorder-1)$-th order derivatives (including $0$-th order) at the start and the end of the trajectory uniquely determine the first and last $\trajorder$ control points respectively; (3) {\cb the $\alpha$-th order derivative of the trajectory is contained within the convex hull formed by the $\alpha$-th order control points}, which can be obtained as 
\begin{align}
    \ctrlptbspline_l^{\lrangle{\alpha+1}}=\frac{(\trajorder-\alpha)(\ctrlptbspline^{\lrangle{\alpha}}_{l+1}-\ctrlptbspline^{\lrangle{\alpha}}_l)}{t_{l+\trajorder+1}-t_{l+\alpha+1}}, \forall l\in[0,\lambda-\alpha]\cap\mathbb{Z},\label{eq: bspline}
\end{align}
$\forall\alpha\in[0,\trajorder-1]\cap\mathbb{Z}$, where $\ctrlptbspline_l^{\lrangle{\alpha}}$ denotes the $\alpha$-th order control point and $\ctrlptbspline_l^{\lrangle{0}}=\ctrlptbspline_l$.
Using the above properties, we design an incremental control points adjustment scheme to obtain a feasible Z-axis trajectory.
In our case $\trajorder=3$, we would like to obtain $\maxop{\eta}$ pieces of trajectories, where $\maxop{\eta}$ is user-defined.
{\cb Each trajectory has a time interval $T$ to be consistent with the kinodynamic search, hence $\lambda=\maxop{\eta}+\trajorder-1$.}
We firstly determine the first $3$ control points of the b-spline from the initial states $p^{\text{in},\text{z}}$, $v^{\text{in},\text{z}}$ and $a^{\text{in},\text{z}}$, and set the last $3$ control points equal to the terminal target altitude $p^{\text{term},\text{z}}$.
Then we set the middle points $\ctrlptbspline_3,\dots,\ctrlptbspline_{\lambda-3}$ such that they are equally spaced between $\ctrlptbspline_2$ and $\ctrlptbspline_{\lambda-2}$.
This setting corresponds to an initial trajectory that will start at the given states and reach $p^{\text{term},\text{z}}$ with zero velocity and acceleration.
Next, we check the dynamic feasibility of this trajectory by computing the velocity and acceleration control points using equation (\ref{eq: bspline}). 
{\cb The lower-order control points are adjusted if the higher-order points exceed the bound.}
For example, given that the acceleration exceeds the bound, $\ctrlptbspline_l^{\lrangle{2}}>\maxop{a}^{\text{z}}$, we adjust the velocity and position control points
\begin{align}
    &\ctrlptbspline_{l+1}^{\lrangle{1}}=\ctrlptbspline_{l}^{\lrangle{1}}+\frac{T}{\trajorder-1}\maxop{a}^{\text{z}},\\
    &\ctrlptbspline_{l+2}=\ctrlptbspline_{l+1}+\frac{T}{\trajorder}\ctrlptbspline_{l+1}^{\lrangle{1}},
\end{align}
so that the updated acceleration is within bound $\ctrlptbspline_l^{\lrangle{2}}=\maxop{a}^{\text{z}}$. 
{\cb The output trajectory satisfies all dynamic constraints while ending at a state as close to $p^{\text{term},\text{z}}$ as possible.}
Finally, we convert the b-spline control points into polynomial coefficients using the basis matrices described in \cite{qin2000general}.

\end{appendices}

\bibliographystyle{ieeetr}        % Include this if you use bibtex 
\bibliography{main} 

\begin{IEEEbiography}[{\includegraphics[width=1in,height=1.25in,clip,keepaspectratio]{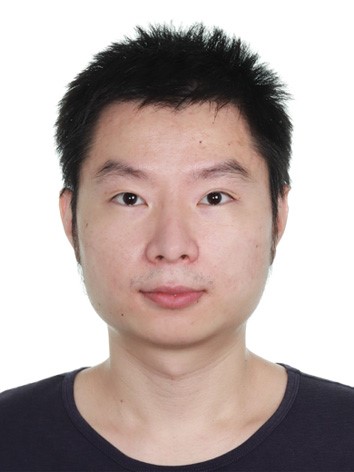}}]{Muqing Cao} received his Bachelor of Engineering (Honors) in Aerospace Engineering from Nanyang Technological University, Singapore. He is currently a Ph.D. candidate with the School of Electrical and Electronic Engineering, Nanyang Technological University, Singapore. 
He is also a research officer in Delta-NTU Corporate Laboratory for Cyber-Physical Systems, working on robot localization and navigation in environments with high human traffic.
His research interests include multi-robot systems, tethered robots, motion planning, modeling and control of aerial and ground robots.
\end{IEEEbiography}

\begin{IEEEbiography}[{\includegraphics[width=1in,height=1.25in,clip,keepaspectratio]{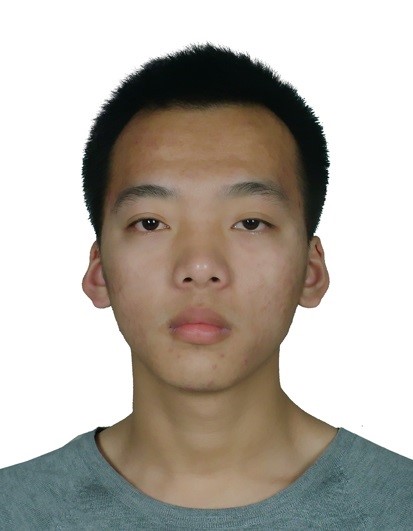}}]{Kun Cao} received the B.Eng. degree in Mechanical engineering from the Tianjin University, Tianjin, China, in 2016, and the Ph.D. degree in the School of Electrical and Electronic Engineering, Nanyang Technological University, Singapore, in 2021. He is currently the 2022 Wallenberg-NTU Presidential Postdoctoral Fellow with the latter school. His research interests include localization, formation control, distributed optimization, and soft robotics.
\end{IEEEbiography}

\begin{IEEEbiography}[{\includegraphics[width=1in,height=1.25in,clip,keepaspectratio]{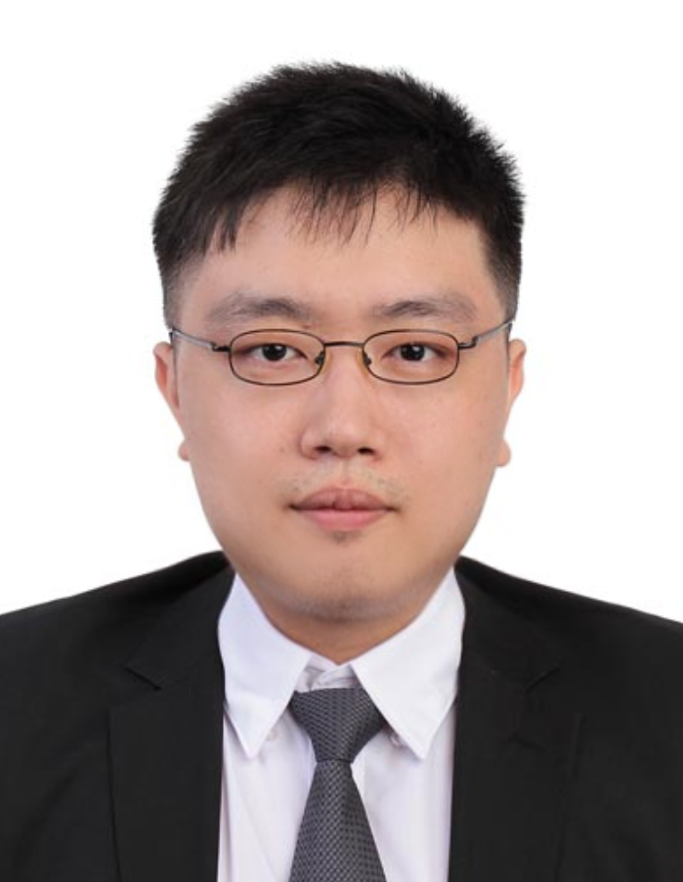}}]{Shenghai Yuan}  obtained his Bachelor's and Ph.D. degrees in electrical and electronic engineering from The Nanyang Technological University in Singapore in 2013 and 2019, respectively. He currently serves as a Research Fellow at the Centre for Advanced Robotics Technology Innovation (CARTIN) at Nanyang Technological University.

The main focus of his research lies in perception, sensor fusion, robust navigation, machine learning, and autonomous robotics systems. He has participated in various robotics competitions and obtained a championship win in the 2011 Taiwan UAV Reconnaissance Competition, and was the finalist in the 2012 DAPRA UAVforge Challenges. Additionally, he was awarded the NTU graduate scholarship in 2013. He has also authored six patents and technological disclosures. His research work has been published in over 40 international conferences and journals.
\end{IEEEbiography}

\begin{IEEEbiography}[{\includegraphics[width=1in,height=1.25in,clip,keepaspectratio]{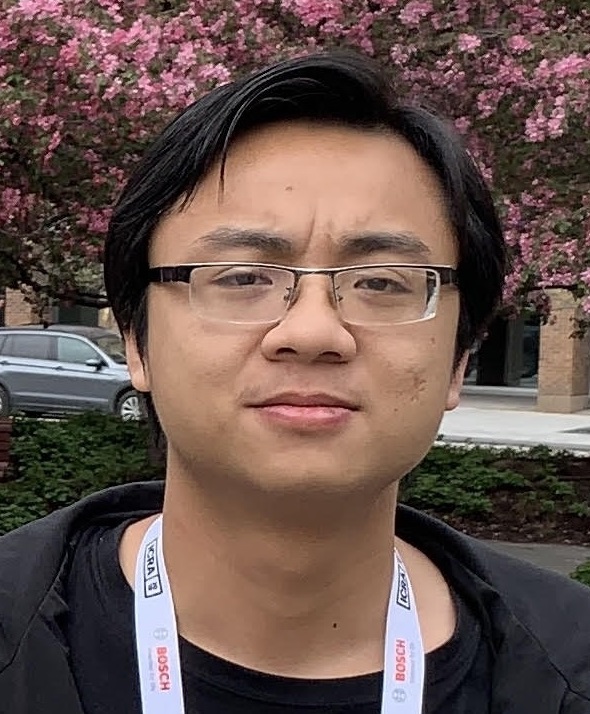}}]{Thien-Minh Nguyen}
received his B.E. (Honors) in Electrical and Electronic Engineering from Vietnam National University - Ho Chi Minh City in 2014, and Ph.D. degree from Nanyang Technological University (NTU) in 2020. He was a Research Fellow under STE-NTU Corporate Lab from Sep 2019 to Nov 2020. He is currently the 2020 Wallenberg-NTU Presidential Postdoctoral Fellow at the School of EEE, NTU, Singapore.

Dr. Nguyen's research interests include perception and control for autonomous robots, learning and adaptive systems, multi-robot systems. He received the NTU EEE Best Doctoral Thesis Award (Innovation Category) in 2020, 1st Prize for Vietnam Go Green in the City competition in 2013, and Intel Engineering Scholarship in 2012.
\end{IEEEbiography}

\begin{IEEEbiography}[{\includegraphics[width=1in,height=1.25in,clip,keepaspectratio]{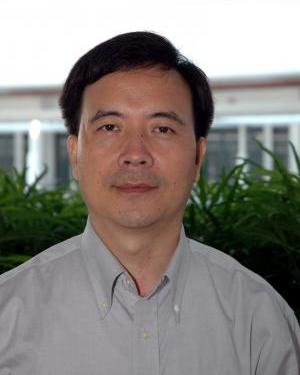}}]{Lihua Xie}
received the Ph.D. degree in electrical engineering from the University of Newcastle, Australia, in 1992. Since 1992, he has been with the School of Electrical and Electronic Engineering, Nanyang Technological University, Singapore, where he is currently a professor and Director, Delta-NTU Corporate Laboratory for Cyber-Physical Systems and Director, Center for Advanced Robotics Technology Innovation. He served as the Head of Division of Control and Instrumentation from July 2011 to June 2014. He held teaching appointments in the Department of Automatic Control, Nanjing University of Science and Technology from 1986 to 1989. 

Dr Xie’s research interests include robust control and estimation, networked control systems, multi-agent networks, localization and unmanned systems. He is an Editor-in-Chief for Unmanned Systems and has served as Editor of IET Book Series in Control and Associate Editor of a number of journals including IEEE Transactions on Automatic Control, Automatica, IEEE Transactions on Control Systems Technology, IEEE Transactions on Network Control Systems, and IEEE Transactions on Circuits and Systems-II. He was an IEEE Distinguished Lecturer (Jan 2012 – Dec 2014). Dr Xie is Fellow of Academy of Engineering Singapore, IEEE, IFAC, and CAA.
\end{IEEEbiography}
\end{document}